\documentclass{ecai} 

\usepackage{latexsym}
\usepackage{amssymb}
\usepackage{amsmath}
\usepackage{amsthm}
\usepackage{booktabs}
\usepackage{enumitem}
\usepackage{graphicx}
\usepackage{color}
\usepackage{mathtools}
\usepackage{tabularray}
\usepackage{nicematrix}

\usepackage{tikzit}

\tikzstyle{node}=[fill=white, draw=black, shape=circle, minimum size=1mm, ultra thick]
\tikzstyle{small box}=[fill=white, draw=black, shape=rectangle, minimum height=0.5cm, minimum width=0.5cm, ultra thick]
\tikzstyle{weyl}=[fill=white, draw={rgb,255: red,0; green,0; blue,109}, shape=rectangle, minimum height=0.5cm, minimum width=0.5cm]
\tikzstyle{filled_node}=[fill=black, draw=black, shape=circle, minimum size=1mm, ultra thick]

\tikzstyle{thick}=[-, ultra thick]
\tikzstyle{blue_thick}=[-, ultra thick, draw=blue]
\tikzstyle{dashes}=[-, dashed, draw={rgb,255: red,191; green,191; blue,191}, dash pattern=on 2mm off 1mm, fill={rgb,255: red,244; green,228; blue,0}]
\tikzstyle{thick_arrow}=[ultra thick, ->]
\tikzstyle{dash_1}=[-, dashed]
\tikzstyle{dash_2}=[-, dashed, fill={rgb,255: red,246; green,235; blue,255}]
\tikzstyle{dash_3}=[-, dashed, fill={rgb,255: red,229; green,255; blue,181}]
\tikzstyle{dash_4}=[-, dashed, fill={rgb,255: red,255; green,209; blue,153}]
\tikzstyle{red_thick}=[-, ultra thick, draw=red]
\tikzstyle{dash_5}=[-, dashed, fill={rgb,255: red,225; green,255; blue,254}]

\usepackage{graphicx}
\usepackage[most]{tcolorbox}
\usepackage{bm}

\usepackage{listings}
\usepackage{breakurl}

\newtheorem{theorem}{Theorem}

\newtheorem{proposition}[theorem]{Proposition}

\newtheorem{definition}[theorem]{Definition}
\theoremstyle{remark}
\newtheorem{remark}[theorem]{Remark}
\newtheorem{example}[theorem]{Example}

\newcommand{\BibTeX}{B\kern-.05em{\sc i\kern-.025em b}\kern-.08em\TeX}
\DeclareMathOperator{\End}{End}

\DeclareMathOperator{\Hom}{Hom}

\DeclareMathOperator{\Bell}{B}

\begin{document}


\begin{frontmatter}


\paperid{264} 


\title{
	Connecting Permutation Equivariant Neural Networks and Partition Diagrams
}


\author[A]{\fnms{Edward}~\snm{Pearce-Crump}\thanks{Corresponding Author. Email: ep1011@ic.ac.uk}}

	\address[A]{Imperial College London, United Kingdom}


\begin{abstract}
	Permutation equivariant neural networks are often constructed 
	using tensor powers of $\mathbb{R}^{n}$
	as their layer spaces.
	We show that
	all of the weight matrices 
	that appear in these neural networks
	can be obtained from
	Schur--Weyl duality between the symmetric group and the partition algebra.
	In particular, we adapt
	Schur--Weyl duality 
	to derive a simple, diagrammatic method for calculating the weight matrices themselves.

\end{abstract}

\end{frontmatter}

\section{Introduction}

Encoding permutation symmetries into neural networks has proven 
to be very useful for performing a large number of machine learning tasks.
The use cases range from standard examples 
such as learning from sets \cite{deepsets} and graphs \cite{morris22a}
through to predicting dynamics of objects in computer vision \cite{guttenberg2016}, 
modelling composition in natural language \cite{gordon2020}, and even
designing auctions that maximise expected revenues in economics \cite{rahme}.

Existing work on permutation equivariant neural networks 
using tensor power spaces of $\mathbb{R}^n$ as their layers
has focused on two main areas: 
designing networks that encode permutation symmetries on sets of data 
for specific 
applications,
and creating 
more general
permutation equivariant functions
for learning from data 
that lives on higher-order structures, such as graphs.
For the former,
\citet{qi2017pointnet} constructed a permutation equivariant neural network to learn from
point cloud data.
\citet{deepsets} developed a permutation equivariant neural network
to learn from sets of data, and used it for image tagging and set anomaly detection tasks. 
\citet{Hartford} modelled interactions between different sets of objects
using a permutation equivariant neural network.
For the latter, 
\citet{hy2019covariant} 
considered higher order relations between sets of indices instead,
and showed that a number of operations on the resulting tensor power spaces of $\mathbb{R}^n$ are permutation equivariant.
\citet{maron2018} then
studied
the problem of classifying all of the linear permutation equivariant and invariant neural network layer functions on tensor power spaces of $\mathbb{R}^n$, with their motivation coming from learning relations between the nodes of graphs. 
They characterised all of the learnable, linear, permutation equivariant layer functions 
from a $k$-order tensor of $\mathbb{R}^n$ to an $l$-order tensor of $\mathbb{R}^n$
in the practical cases (specifically, when $n \geq k+l$). 
Their method used equalities involving Kronecker products to obtain a number of fixed point equations which they then solved to find a basis, in tensor form, for the layer functions under consideration.
\citet{pan22} went on to establish a method for organising the computation of the layer functions that appeared in \citet{maron2018}, and applied it to the task of predicting the efficacy of certain drug combinations.
\citet{finzi} developed a numerical algorithm to calculate the weight matrices for permutation and other group equivariant neural networks for small values of $n$, $k$ and $l$.

In this paper, we show that an entirely different approach 
from the one that appears in \citet{maron2018} can be used
to obtain a full characterisation 
of all of the possible permutation equivariant weight matrices 
that appear between any two tensor power spaces of $\mathbb{R}^n$. 
The starting point for our approach is 
Schur--Weyl duality, 
a result that 
commonly appears in 
the algebraic combinatorics and representation theory literature
\cite{BenHal1, BenHal2, BHH, HalRam, Jones, Martin0, Martin1, Martin2}.
We describe Schur--Weyl duality in more detail in the next section.
Duality itself appears in many areas of mathematics and physics \cite{atiyah2007duality} 
as a concept for understanding one object through two different viewpoints.
In this paper, we show that the weight matrices, which have permutation symmetry --- 
the first viewpoint --- can be obtained analytically through a so-called partition 
vector space consisting of combinatorial diagrams that partition sets 
into disjoint subsets --- the second viewpoint.

Schur--Weyl duality has proven to be the cornerstone of many of the results that 
have appeared recently in the quantum machine learning literature 
\cite{east2023, larocca, nguyen2022, ragone2022, schatzki2024, zheng}.
It has only recently appeared in the ``classical" machine learning literature \cite{pearcecrumpB}
where it was used to fully characterise the weight matrices 
that appear between any two tensor power spaces of $\mathbb{R}^n$ for three compact groups.
With our contribution, we add to this growing body of work that shows that 
Schur--Weyl duality is a powerful principle for constructing group equivariant 
neural network architectures.





\section{Schur--Weyl Duality} \label{schurweyl}

Schur--Weyl duality is a result that first appeared 
in a paper written in 1927 by Issai Schur \cite{schur1927}; however, this result was mostly 
a reformulation of his own ideas that appeared 
in his doctoral thesis of 1901 in a different form \cite{schur1901}.
In spite of this, Schur--Weyl duality only became well-known 
through the work of
Hermann Weyl \cite{weyl}.
Schur wanted to understand all of the irreducible representations of the general linear group
$GL_n$. 
He lived at a time where the irreducibles of the symmetric group had been characterised 
by Young \cite{young}
in the years preceding his own contribution.
Young showed that the irreducibles of the symmetric group $S_n$ correspond bijectively with
all possible integer partitions of $n$.
Schur used this result to establish a one-to-one correspondence between the
irreducibles of the general linear group $GL_n$ 
and the irreducibles of the symmetric group $S_k$
that appear in the decomposition of 
the tensor power space
$(\mathbb{R}^{n})^{\otimes k}$, namely
\begin{equation} \label{GLnSWduality}
	(\mathbb{R}^{n})^{\otimes k}
	\cong \bigoplus_{\lambda \in \Lambda(k,n)} 
	V_{n}^{\lambda}
	\otimes 
	S_k^{\lambda} 
\end{equation}
In (\ref{GLnSWduality}),
the irreducibles $V_{n}^{\lambda}$ of the general linear group $GL_n$ are indexed 
by the same integer partitions $\lambda$ of $k$ into at most $n$ parts 
that index irreducibles $S_k^{\lambda}$ of the symmetric group $S_k$.
It is this result that became known as Schur--Weyl duality.

However, a number of other Schur--Weyl dualities have appeared since Schur's discovery
\cite{Brauer, HalRam2022, Jones, Martin0, Martin1, Martin2}.
The Schur--Weyl duality that is
the focus of this paper is the
one that exists between
the symmetric group $S_n$ and
the partition algebra $P_k^k(n)$
that was simultaneously found by 
\citet{Martin0, Martin1, Martin2} and \citet{Jones}.

Before stating what this Schur--Weyl duality is, we need to define the partition algebra.
To do this, we require the following two definitions.
We write $[n]$ to represent the set $\{1, \dots, n\}$ throughout this paper.

\begin{definition}
	A \textbf{set partition} $\pi$ of $[2k]$ is
	a partition of the set $[2k]$ into a number of disjoint subsets.
	We call the subsets of $\pi$ \textbf{blocks}.
\end{definition}

\begin{definition}
	We define a \textbf{diagram} $d_\pi$ from each set partition $\pi$ of $[2k]$
	that has two rows of vertices and edges between vertices such that there are
	\begin{enumerate}
		\item $k$ black vertices on the top row, labelled by $1, \dots, k$
		\item $k$ black vertices on the bottom row, labelled by $k+1, \dots, 2k$, and
		\item the edges between the vertices correspond to the connected components of the set partition $\pi$ that indexes the diagram.
	\end{enumerate}
\end{definition}

Consequently, we have that 
\begin{definition} \label{partalgebra}
	The \textbf{partition algebra} $P_k^k(n)$ is 
	the $\mathbb{R}$-linear span of the set of diagrams
	$d_\pi$ indexed by all of the set partitions $\pi$ of $[2k]$ 
	(together with an algebra product that we omit for brevity).
\end{definition}

Similar to Schur's 1927 version,
Schur--Weyl duality 
between
the symmetric group $S_n$ and
the partition algebra $P_k^k(n)$
describes a one-to-one correspondence between the
irreducibles of the symmetric group $S_n$
and the irreducibles of the partition algebra $P_k^k(n)$
that appear in the decomposition of 
the tensor power space
$(\mathbb{R}^{n})^{\otimes k}$, namely
\begin{equation} \label{PartSWduality}
	(\mathbb{R}^{n})^{\otimes k}
	\cong \bigoplus_{\lambda \in \Lambda(n)} S_n^{\lambda} \otimes Z_{k,n}^{\lambda}
\end{equation}
Here, $\Lambda(n)$ is the set of all integer partitions of $n$, 
$S_n^{\lambda}$ is an irreducible of the symmetric group $S_n$ and
$Z_{k,n}^{\lambda}$ is an irreducible of the partition algebra
$P_k^k(n)$.

This Schur--Weyl duality was obtained by \citet{Jones}
through a surjective map
from the partition algebra
$P_k^k(n)$ onto $\End_{S_n}((\mathbb{R}^{n})^{\otimes k})$.
We describe this map in what follows as we adapt it, and
hence Schur--Weyl duality,
to characterise all of the possible weight matrices that
can appear in permutation equivariant neural networks where the layers are some 
tensor power of $\mathbb{R}^n$.



%


\section{Characterisation of Permutation Equivariant Linear Layer Functions} \label{charPerm}


Many permutation equivariant neural networks are constructed by alternately composing
linear and non-linear equivariant functions between layer spaces that are a 
tensor power of $\mathbb{R}^{n}$ \cite{lim}. 
These layer spaces are representations of the symmetric group $S_n$ in the following sense.

Recall that 
$\mathbb{R}^{n}$ is a representation of $S_n$, called the permutation representation, via its action on the standard basis $\{e_a \mid a \in [n]\}$ which is extended linearly.
Specifically, the action is given by
\begin{equation}
	\sigma \cdot e_a = e_{\sigma(a)} \text{ for all } \sigma \in S_n \text{ and } a \in [n]
\end{equation}
Consequently, for any positive integer $k$, 
the $k$-tensor power of the permutation representation, 
$(\mathbb{R}^{n})^{\otimes k}$, 
is a representation of $S_n$
since the elements 
\begin{equation} \label{tensorelementfirst}
	e_I \coloneqq e_{i_1} \otimes e_{i_2} \otimes \dots \otimes e_{i_k} 
\end{equation}
for all $I \coloneqq (i_1, i_2, \dots, i_k) \in [n]^k$ form the standard basis of 
$(\mathbb{R}^{n})^{\otimes k}$, 
and the action of $S_n$ that maps a basis element of 
$(\mathbb{R}^{n})^{\otimes k}$
of the form (\ref{tensorelementfirst}) to
\begin{equation}
	e_{\sigma(I)} \coloneqq e_{\sigma(i_1)} \otimes e_{\sigma(i_2)} \otimes \dots \otimes e_{\sigma(i_k)} 
\end{equation}
can be extended linearly. 
We denote the representation itself by $\rho_{k}$.

Moreover, a permutation equivariant function between two tensor power spaces is defined as follows.
\begin{definition} \label{equivariance}
	A map $\phi : 
	(\mathbb{R}^{n})^{\otimes k} \rightarrow
	(\mathbb{R}^{n})^{\otimes l}$
	is said to be \textbf{permutation equivariant} if,
	for all $\sigma \in S_n$ and $v \in 
	(\mathbb{R}^{n})^{\otimes k}$,
	\begin{equation} \label{equivmapdefn}
		\phi(\rho_{k}(\sigma)[v]) = \rho_{l}(\sigma)[\phi(v)]
	\end{equation}
	We denote the set of all \textit{linear} permutation equivariant maps 
	between $(\mathbb{R}^{n})^{\otimes k}$ and $(\mathbb{R}^{n})^{\otimes l}$ by
	\begin{equation} \label{desiredhomspace}
		\Hom_{S_n}((\mathbb{R}^{n})^{\otimes k},(\mathbb{R}^{n})^{\otimes l})
	\end{equation}
	It can be shown that 
	(\ref{desiredhomspace})
	is a vector space over $\mathbb{R}$. 
	See \citet{segal} for more details.
	Note that 
	(\ref{desiredhomspace})
	is a subspace of 
	$\Hom((\mathbb{R}^{n})^{\otimes k}, (\mathbb{R}^{n})^{\otimes l})$, 
	the vector space of all linear maps from 
	$(\mathbb{R}^{n})^{\otimes k}$ 
	to $(\mathbb{R}^{n})^{\otimes l}$.
\end{definition}

Our goal is 
to calculate all of the weight matrices
that can appear between any two layers of the permutation equivariant neural networks in question.
It is enough to 
construct a basis of matrices for 
$\Hom_{S_n}((\mathbb{R}^{n})^{\otimes k}, (\mathbb{R}^{n})^{\otimes l})$,
by viewing it as a subspace of 
$\Hom((\mathbb{R}^{n})^{\otimes k}, (\mathbb{R}^{n})^{\otimes l})$
and choosing the standard basis of
$\mathbb{R}^{n}$,
since any weight matrix will be a 
weighted linear combination of these basis matrices.

To construct such a basis, we begin by introducing the following vector spaces 
that are adapted from 
	the definition of
the partition algebra that
appeared in Section \ref{schurweyl}.


\subsection{The Partition Vector Space, $P_k^l(n)$} \label{partitionalgebra}


Instead of considering the set $[2k]$, we now look at
the set $[l+k]$.
As before, we can create a set partition of $[l+k]$ by partitioning it 
into a number of disjoint subsets,
which we also call blocks.
Let $\Pi_{l+k}$ be the set of all set partitions of $[l+k]$.
It will also be useful to define the set $\Pi_{l+k, n}$, which is the subset of $\Pi_{l+k}$ consisting of all set partitions of $[l+k]$ 
having at most $n$ blocks.

As the number of set partitions in $\Pi_{l+k}$ having exactly $t$ blocks is the Stirling number
$
\begin{Bsmallmatrix}
l+k\\
t 
\end{Bsmallmatrix}
$
of the second kind, we see that the number of elements in $\Pi_{l+k}$ is equal to $\Bell(l+k)$, the $(l+k)^{\text{th}}$ Bell number, and that the number of elements in $\Pi_{l+k,n}$ is therefore equal to
\begin{equation}
	\sum_{t=1}^{n} 
		\begin{Bmatrix}
		l+k\\
		t 
		\end{Bmatrix}
	= \Bell(l+k,n)
\end{equation}
the $n$-restricted $(l+k)^{\text{th}}$ Bell number.

\begin{example} \label{setpartexample}
	If $l = 4$ and $k = 5$, then
	\begin{equation} \label{setpartex1}
		\pi := \{1, 2, 5, 7 \mid 3, 4, 8 \mid 6 \mid 9\}
	\end{equation}
	is a set partition in $\Pi_{4+5}$ with $4$ blocks. 
	Hence $\pi \in \Pi_{4+5,n}$ for all $n \geq 4$.
\end{example}

Similar to the partition algebra, we
can form a vector space from the $\mathbb{R}$-linear span of a set of diagrams $d_\pi$, except this time they are indexed by the elements $\pi$ of $\Pi_{l+k}$.
Each diagram $d_\pi$ in the set has two rows of vertices and edges between vertices, except now
there are
\begin{enumerate}
	\item $l$ black vertices on the top row, labelled by $1, \dots, l$
	\item $k$ black vertices on the bottom row, labelled by $l+1, \dots, l+k$, and
	\item the edges between the vertices correspond to the connected components of the set partition $\pi$ that indexes the diagram.
\end{enumerate}
As a result, $d_{\pi}$ represents the equivalence class of all diagrams with connected components equal to the blocks of $\pi$. 
We call this vector space the \textbf{partition vector space}, and denote it by
$P_k^l(n)$. By construction, it has dimension $\Bell(l+k)$.
We call the basis described here the diagram basis.

\begin{example} 
	Continuing on from Example \ref{setpartexample},
	we see that the diagram $d_{\pi}$ 
	corresponding to the set partition $\pi$ 
	given in (\ref{setpartex1}) is
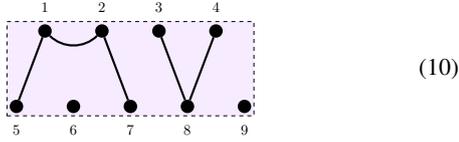
	\begin{equation}	
		\begin{aligned}
			\scalebox{0.5}{\begin{tikzpicture}
	\begin{pgfonlayer}{nodelayer}
		\node [style=none] (5) at (2.5, 1.25) {\large$1$};
		\node [style=none] (8) at (0.5, 0.5) {};
		\node [style=none] (9) at (0.5, -4.5) {};
		\node [style=none] (10) at (13.5, -4.5) {};
		\node [style=none] (11) at (13.5, 0.5) {};
		\node [style={filled_node}] (20) at (2.5, 0) {};
		\node [style={filled_node}] (21) at (5.5, 0) {};
		\node [style={filled_node}] (22) at (8.5, 0) {};
		\node [style={filled_node}] (23) at (11.5, 0) {};
		\node [style={filled_node}] (24) at (1, -4) {};
		\node [style={filled_node}] (25) at (4, -4) {};
		\node [style={filled_node}] (26) at (7, -4) {};
		\node [style={filled_node}] (27) at (10, -4) {};
		\node [style={filled_node}] (28) at (13, -4) {};
		\node [style=none] (30) at (5.5, 1.25) {\large$2$};
		\node [style=none] (31) at (8.5, 1.25) {\large$3$};
		\node [style=none] (32) at (11.5, 1.25) {\large$4$};
		\node [style=none] (33) at (1, -5.25) {\large$5$};
		\node [style=none] (34) at (4, -5.25) {\large$6$};
		\node [style=none] (35) at (7, -5.25) {\large$7$};
		\node [style=none] (36) at (10, -5.25) {\large$8$};
		\node [style=none] (37) at (13, -5.25) {\large$9$};
	\end{pgfonlayer}
	\begin{pgfonlayer}{edgelayer}
		\draw [style={dash_2}] (9.center)
			 to (10.center)
			 to (11.center)
			 to (8.center)
			 to cycle;
		\draw [style=thick] (21) to (26);
		\draw [style=thick] (23) to (27);
		\draw [style=thick] (27) to (22);
		\draw [style=thick] (20) to (24);
		\draw [style=thick, bend right=45] (20) to (21);
	\end{pgfonlayer}
\end{tikzpicture}}
		\end{aligned}
	\end{equation}
\end{example}

\begin{remark}
	It is clear that if we set $l = k$, then we obtain the partition algebra $P_k^k(n)$
	that was given in Definition \ref{partalgebra}.
\end{remark}



\subsection{The Orbit Basis of $P_k^l(n)$}

We can construct another basis of $P_k^l(n)$ that we will use in what follows
to obtain the basis of matrices for 
$\Hom_{S_n}((\mathbb{R}^{n})^{\otimes k}, (\mathbb{R}^{n})^{\otimes l})$.

First, we define a partial ordering on the set partitions in $\Pi_{l+k}$,
denoted by $\preceq$, which states that, for all $\pi_1$, $\pi_2 \in \Pi_{l+k}$, 
	$\pi_1 \preceq \pi_2$
	if every block of $\pi_1$ is contained in a block of $\pi_2$.
	
Then we can define a set of elements in $P_k^l(n)$ indexed by the set partitions of $\Pi_{l+k}$, 
$B_O \coloneqq \{x_\pi \mid \pi \in \Pi_{l+k}\}$, 
with respect to the diagram basis as
\begin{equation} \label{orbitbasis}
	d_\pi = \sum_{\pi \preceq \theta} x_\theta
\end{equation}
To see why the set $B_O$ forms a basis of $P_k^l(n)$, first,
we form an ordered set of set partitions of $\Pi_{l+k}$ by ordering the set partitions by the number of blocks that they have from smallest to largest, with any arbitrary ordering allowed for a pair of set partitions that have the same number of blocks. 
Call this set $S_{{l+k}}$.
Then, because the square matrix that maps elements of the diagram basis to linear combinations of the set $B_O$ --- whose rows and columns are indexed (in order) by the ordered set $S_{{l+k}}$ --- is unitriangular by (\ref{orbitbasis}), it is therefore invertible, and so we get that
$B_O$ forms a basis of $P_k^l(n)$.
We call $B_O$ the orbit basis of $P_k^l(n)$.


For each set partition $\pi \in \Pi_{l+k}$, we represent its corresponding orbit basis element $x_\pi$
as a diagram in the same way as $d_\pi$, except we use white vertices in each row of the diagram instead.


\begin{example} \label{exP11n}
	The orbit basis of $P_1^1(n)$ consists of the two elements
	\begin{equation} \label{orbitbasis11}
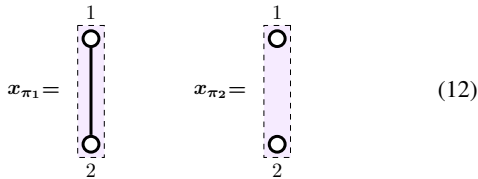
		\begin{aligned}
			\scalebox{0.7}{\begin{tikzpicture}
	\begin{pgfonlayer}{nodelayer}
		\node [style=node] (1) at (3, 0) {};
		\node [style=node] (2) at (3, -4) {};
		\node [style=none] (4) at (2.5, 0.5) {};
		\node [style=none] (5) at (2.5, -4.5) {};
		\node [style=none] (6) at (3.5, -4.5) {};
		\node [style=none] (7) at (3.5, 0.5) {};
		\node [style=node] (9) at (10, 0) {};
		\node [style=node] (10) at (10, -4) {};
		\node [style=none] (12) at (9.5, 0.5) {};
		\node [style=none] (13) at (9.5, -4.5) {};
		\node [style=none] (14) at (10.5, -4.5) {};
		\node [style=none] (15) at (10.5, 0.5) {};
		\node [style=none] (16) at (1.5, -2) {};
		\node [style=none] (17) at (1.5, -2) {\large$\bm{=}$};
		\node [style=none] (18) at (8.5, -2) {\large$\bm{=}$};
		\node [style=none] (19) at (7.5, -2) {\large$\bm{x_{\pi_2}}$};
		\node [style=none] (20) at (0.5, -2) {\large$\bm{x_{\pi_1}}$};
		\node [style=none] (21) at (3, -5) {\large$2$};
		\node [style=none] (22) at (10, -5) {\large$2$};
		\node [style=none] (23) at (3, 1) {\large$1$};
		\node [style=none] (24) at (10, 1) {\large$1$};
	\end{pgfonlayer}
	\begin{pgfonlayer}{edgelayer}
		\draw [style={dash_2}] (4.center)
			 to (5.center)
			 to (6.center)
			 to (7.center)
			 to cycle;
		\draw [style=thick] (2) to (1);
		\draw [style={dash_2}] (15.center)
			 to (12.center)
			 to (13.center)
			 to (14.center)
			 to cycle;
	\end{pgfonlayer}
\end{tikzpicture}}
		\end{aligned}
	\end{equation}
	Hence, any element of $P_1^1(n)$ can be expressed as
	\begin{equation} \label{exP11nlincomb}
		\begin{aligned}
			\scalebox{0.7}{\begin{tikzpicture}
	\begin{pgfonlayer}{nodelayer}
		\node [style=none] (0) at (3, 1) {\large$1$};
		\node [style=node] (1) at (3, 0) {};
		\node [style=node] (2) at (3, -4) {};
		\node [style=none] (3) at (3, -5) {\large$2$};
		\node [style=none] (4) at (2.5, 0.5) {};
		\node [style=none] (5) at (2.5, -4.5) {};
		\node [style=none] (6) at (3.5, -4.5) {};
		\node [style=none] (7) at (3.5, 0.5) {};
		\node [style=node] (9) at (7, 0) {};
		\node [style=node] (10) at (7, -4) {};
		\node [style=none] (12) at (6.5, 0.5) {};
		\node [style=none] (13) at (6.5, -4.5) {};
		\node [style=none] (14) at (7.5, -4.5) {};
		\node [style=none] (15) at (7.5, 0.5) {};
		\node [style=none] (17) at (4.5, -2) {\large$\bm{+}$};
		\node [style=none] (19) at (5.5, -2) {\large$\bm{\lambda_{2}}$};
		\node [style=none] (20) at (1.5, -2) {\large$\bm{\lambda_{1}}$};
		\node [style=none] (21) at (7, 1) {\large$1$};
		\node [style=none] (22) at (7, -5) {\large$2$};
	\end{pgfonlayer}
	\begin{pgfonlayer}{edgelayer}
		\draw [style={dash_2}] (4.center)
			 to (5.center)
			 to (6.center)
			 to (7.center)
			 to cycle;
		\draw [style=thick] (2) to (1);
		\draw [style={dash_2}] (15.center)
			 to (12.center)
			 to (13.center)
			 to (14.center)
			 to cycle;
	\end{pgfonlayer}
\end{tikzpicture}}
		\end{aligned}
	\end{equation}
	for scalars $\lambda_1, \lambda_2 \in \mathbb{R}$.
\end{example}


For more details on the orbit basis, specifically for how to express an orbit basis element as a linear combination of diagram basis elements, see \citet{BenHal1, BenHal2}. 


\subsection{$P_k^l(n)$ and
	a Basis of 
	$\Hom_{S_n}((\mathbb{R}^{n})^{\otimes k}, (\mathbb{R}^{n})^{\otimes l})$}
\label{jonesargument}



In this section, we show how the weight matrices that appear in the permutation equivariant 
neural networks in question are related to the partition vector space $P_k^l(n)$, 
namely by establishing a bijective correspondence between a basis of matrices 
for the vector space of $S_n$-equivariant linear maps from 
$(\mathbb{R}^{n})^{\otimes k}$
to
$(\mathbb{R}^{n})^{\otimes l}$,
expressed in the standard basis of $\mathbb{R}^n$,
and certain orbit basis diagrams that appear in $P_k^l(n)$.

We begin by establishing the following bijective correspondence.

\begin{figure*}[tb]
	\begin{tcolorbox}[colback=teal!10, colframe=teal!30, coltitle=black, 
		title={\bfseries Procedure 1: How to Calculate
		the Weight Matrix of
		an $S_n$-Equivariant Linear Layer Function from 
	$(\mathbb{R}^{n})^{\otimes k}$
	to 
	$(\mathbb{R}^{n})^{\otimes l}$.},
	fonttitle=\bfseries]
		Perform the following steps:
	\begin{enumerate}
		\item Calculate all of the set partitions $\pi$ of $\{1, \dots, l+k\}$ that have
			at most $n$ blocks. 
	\item Express each set partition $\pi$ as an orbit basis diagram $x_\pi$ in $P_k^l(n)$. 
	\item Apply the function $\Phi_{k,n}^l$ to 
		each orbit basis diagram $x_\pi$ to obtain its associated basis matrix $X_\pi$.
	\item Attach a weight $\lambda_\pi \in \mathbb{R}$ to each matrix $X_\pi$.
	\item Finally, calculate $\sum \lambda_\pi X_\pi$ to give the overall weight matrix.
	\end{enumerate}
	Consequently, all of the orbit basis diagrams in $P_k^l(n)$ having at most $n$ blocks
	determine the weight matrix of 
	an $S_n$-equivariant linear layer function from 
	$(\mathbb{R}^{n})^{\otimes k}$
	to 
	$(\mathbb{R}^{n})^{\otimes l}$.
	\end{tcolorbox}
  	\label{summaryprocedure}
\end{figure*}


\begin{proposition} \label{basisorbits}
	The basis elements of 
	$\Hom_{S_n}((\mathbb{R}^{n})^{\otimes k}, (\mathbb{R}^{n})^{\otimes l})$
	are in bijective correspondence with the orbits coming from the action of 
	$S_n$ on the $(l+k)-$fold Cartesian product set $[n]^{l+k}$.
\end{proposition}

\begin{proof}
	As a result of choosing the standard basis for each copy of 
	$\mathbb{R}^{n}$
	that appears in the vector space of all linear maps from
	$(\mathbb{R}^{n})^{\otimes k}$
	to 
	$(\mathbb{R}^{n})^{\otimes l}$,
	this vector space
	has a standard basis of matrix units
	\begin{equation}
		\{E_{I,J}\}_{I \in [n]^l, J \in [n]^k}
	\end{equation}
	where $E_{I,J}$ has a $1$ in the $(I,J)$ position and is $0$ elsewhere.

	Hence, for any standard basis element $e_P \in (\mathbb{R}^{n})^{\otimes k}$,
	we see that
	\begin{equation}
		E_{I,J}e_P = \delta_{J,P}e_I
	\end{equation}
	and so, for any linear map $f: (\mathbb{R}^{n})^{\otimes k} \rightarrow (\mathbb{R}^{n})^{\otimes l}$,
	expressing $f$ in the basis of matrix units as
	\begin{equation}
		f = \sum_{I \in [n]^l}\sum_{J \in [n]^k} f_{I,J}E_{I,J}
	\end{equation}
	we get that
	\begin{equation}
		f(e_P) = \sum_{I \in [n]^l} f_{I,P}e_I
	\end{equation}
	Consequently, given that 
	$\Hom_{S_n}((\mathbb{R}^{n})^{\otimes k}, (\mathbb{R}^{n})^{\otimes l})$
	is a subspace of
	$\Hom((\mathbb{R}^{n})^{\otimes k}, (\mathbb{R}^{n})^{\otimes l})$,
	we have that
	$f$ is an $S_n$-equivariant linear map
	if and only if, 
	for all $\sigma \in S_n$ and standard basis vectors $e_J \in (\mathbb{R}^{n})^{\otimes k}$,
	\begin{equation} \label{equivarg}
		f(\rho_{k}(\sigma)[e_J]) = \rho_{l}(\sigma)[f(e_J)] 
	\end{equation}
	(\ref{equivarg}) holds
	if and only if
	\begin{equation} \label{centeq2}
		\sum_{I \in [n]^l} f_{I,\sigma(J)}e_I = \sum_{I \in [n]^l} f_{I,J}e_{\sigma(I)} 
	\end{equation}
	which is true if and only if
	\begin{equation} \label{centeq1}
		f_{\sigma(I),\sigma(J)} = f_{I,J} 
	\end{equation}
	for all $\sigma \in S_n$, $I \in [n]^{l}$ and $J \in [n]^{k}$.


	Therefore, concatenating the pair $I \in [n]^{l}, J \in [n]^{k}$ into a single element $(I, J) \in [n]^{l+k}$, (\ref{centeq1}) tells us that the basis elements of 
	$\Hom_{S_n}((\mathbb{R}^{n})^{\otimes k}, (\mathbb{R}^{n})^{\otimes l})$
	are in bijective correspondence with the orbits coming from the action of $S_n$ on $[n]^{l+k}$, where $\sigma \in S_n$ acts on the pair $(I,J)$ by 
	\begin{equation} \label{sigmapairaction}
		\sigma(I,J) \coloneqq (\sigma(I),\sigma(J))
	\end{equation}
	However, since $S_n$ acts on $[n]$ transitively, 
	we get that the action of
	$S_n$ on $[n]^{l+k}$ 
	gives a set of orbits that completely
	partition the set $[n]^{l+k}$.
\end{proof}

We now show how the orbits relate to the partition vector space $P_k^l(n)$.

\begin{proposition} \label{orbitsdiags}
	The orbits that come from the action of $S_n$ on $[n]^{l+k}$ 
	are in bijective correspondence with the orbit basis diagrams $x_\pi$
	of $P_k^l(n)$ that have at most $n$ blocks.
\end{proposition}

\begin{proof}
	Consider an orbit coming from the action of $S_n$ on $[n]^{l+k}$.
	We can define the bijection in question 
	on a class representative $(I,J)$ of the orbit as follows.

	Replacing momentarily the elements of $J$ by $i_{l+p} \coloneqq j_p$ for all $p \in [k]$,
	so that
	\begin{align} \label{pairedblocks}
		(I,J) & = (i_1, i_2, \dots, i_l, j_1, j_2, \dots, j_k) \\
		& = (i_1, i_2, \dots, i_l, i_{l+1}, i_{l+2}, \dots i_{l+k})
	\end{align}
	then, for indices $x$, $y \in [l+k]$, we define the bijection by
	\begin{equation} \label{blockbij}
		i_x = i_y \iff x, y \text{ are in the same block of } \pi	
	\end{equation}
	We see that the LHS of (\ref{blockbij}) is checking for an equality on the elements of $[n]$, whereas the RHS is separating the elements of $[l+k]$ into blocks, hence there must be at most $n$ such blocks.

	Moreover, the bijection given in (\ref{blockbij}) is independent of the choice of class representative, since
	\begin{equation} \label{bijindep}
		i_x = i_y \iff \sigma(i_x) = \sigma(i_y) \text{ for all } \sigma \in S_n
	\end{equation}
	This gives us the desired result.
\end{proof}

Combining Propositions \ref{basisorbits} and
\ref{orbitsdiags}, we obtain 
the following key result.

\begin{theorem} \label{thmrestBell}
	For all non-negative integers $l, k$ and positive integers $n$, 
	the basis elements of 
	$\Hom_{S_n}((\mathbb{R}^{n})^{\otimes k}, (\mathbb{R}^{n})^{\otimes l})$
	are in bijective correspondence with
	the orbit basis diagrams $x_\pi$ in $P_k^l(n)$ 
	having at most $n$ blocks, and so
	\begin{equation}
		\dim 
		\Hom_{S_n}((\mathbb{R}^{n})^{\otimes k}, (\mathbb{R}^{n})^{\otimes l})
		=
		\Bell(l+k, n)
	\end{equation}
	where $\Bell(l+k,n)$ is the $n$-restricted $(l+k)^{\text{th}}$ Bell number.
\end{theorem}

\begin{example}
	Suppose that $l = k = 1$, and let $n = 4$.
	It is clear from (\ref{sigmapairaction}) that
	the action of $S_4$ on $[4]^{1+1}$ partitions the set
	into precisely two orbits.
	From (\ref{bijindep}), it is sufficient to choose
	$(1,1)$ to be the class representative of the first orbit 
	and $(1,2)$ to be the class representative of the second orbit.
	(\ref{blockbij}) tells us that the set partition corresponding to the first orbit
	must be $\pi_1 \coloneqq \{1, 2\}$
	whereas the set partition corresponding to the second orbit
	must be $\pi_2 \coloneqq \{1 \mid 2\}$.
	The orbit basis diagrams that correspond to $\pi_1$ and $\pi_2$ 
	first appeared in Example \ref{exP11n}.
	By Theorem \ref{thmrestBell}, the basis matrices
	for the space of $S_4$-equivariant linear maps from 
	$\mathbb{R}^{4}$
	to 
	$\mathbb{R}^{4}$
	correspond bijectively with these orbit basis diagrams, hence there are two of them.
	We show how to calculate the basis matrices in Example
	\ref{114example}.
\end{example}

\subsection{Permutation Equivariant Weight Matrices} \label{permequivmat}



We can go further than Theorem \ref{thmrestBell} and obtain the basis matrices of 
$\Hom_{S_n}((\mathbb{R}^{n})^{\otimes k}, (\mathbb{R}^{n})^{\otimes l})$ themselves
from the orbit basis diagrams in $P_k^l(n)$ having at most $n$ blocks.
In doing so, we show how to construct all of the weight matrices
that can appear between any two tensor power layers of the 
permutation equivariant neural networks in question.



\begin{figure*}[tb]
	\begin{tcolorbox}[colback=blue!10, colframe=blue!30, coltitle=black, 
		title={\bfseries Procedure 2: How to Calculate
		the $(I,J)$-entry of each Permutation Equivariant Basis Matrix $X_\pi$
		from 
		$(\mathbb{R}^{n})^{\otimes k}$
		to 
		$(\mathbb{R}^{n})^{\otimes l}$.},
		fonttitle=\bfseries]
		We assume that $x_\pi$ is an orbit basis diagram in $P_k^l(n)$ having at most $n$ blocks.
		We perform the following steps:
	\begin{enumerate}
	\item Place the indices $I$ on the top row of $x_\pi$ and the indices $J$ on the bottom row of $x_\pi$.
	\item If all of the vertices in each block in $x_\pi$ have been overlaid with the same number, and no two blocks have had their vertices overlaid with the same number, then the $(I, J)$ entry of $X_\pi$ is $1$, otherwise it is $0$.
	\end{enumerate}
	\end{tcolorbox}
  	\label{powermethod}
\end{figure*}



To obtain the basis matrices, we first need to define a procedure for labelling
the blocks of an orbit basis diagram $x_\pi$ in $P_k^l(n)$ having at most $n$ blocks.


\begin{definition}
	Let $x_\pi$ be an orbit basis diagram in $P_k^l(n)$ having at most $n$ blocks.
	Denote the number of blocks in $x_\pi$ by $t$.

	We obtain a \textbf{block labelling} for $x_\pi$
	by letting $B_1$ be the block that contains 
	the number $1 \in [l+k]$, and iteratively letting $B_j$, for $1 < j \leq t$, be the block 
	that contains the smallest number in $[l+k]$ that is not 
	in $B_1 \cup B_2 \cup \dots \cup B_{j-1}$.

	We can represent the block labelling for $x_\pi$ in two equivalent forms. 
	The first form is an $(l+k)-$length tuple $(I_\pi, J_\pi$) with elements in $[n]$, 
	where the length of $I_\pi$ is $l$,
	the length of $J_\pi$ is $k$,
	and the $p^{\text{th}}$ entry is the label of the block that contains vertex $p$. 
	The second form is a diagram which is obtained by relabelling each vertex in 
	the orbit basis diagram $x_\pi$ with the label of the block containing that vertex. 
	We will see that this particular form is very useful in what follows
	as it highlights the structure of the blocks and their labels in the block labelling for $x_\pi$.

\end{definition}

\begin{example} \label{blocklabelex}
	Suppose that we have the orbit basis diagram $x_\pi$
	\begin{equation} \label{blocklabellingvert}
		\begin{aligned}
			\scalebox{0.5}{\begin{tikzpicture}
	\begin{pgfonlayer}{nodelayer}
		\node [style=none] (16) at (10, 1.25) {\large$2$};
		\node [style=none] (18) at (0.5, 0.5) {};
		\node [style=none] (19) at (0.5, -4.5) {};
		\node [style=none] (20) at (16.5, -4.5) {};
		\node [style=none] (21) at (16.5, 0.5) {};
		\node [style=node] (22) at (1, -4) {};
		\node [style=node] (23) at (10, 0) {};
		\node [style=node] (24) at (7, 0) {};
		\node [style=node] (25) at (4, -4) {};
		\node [style=node] (26) at (7, -4) {};
		\node [style=node] (27) at (10, -4) {};
		\node [style=node] (28) at (13, -4) {};
		\node [style=node] (29) at (16, -4) {};
		\node [style=none] (30) at (1, -5.25) {\large$3$};
		\node [style=none] (31) at (4, -5.25) {\large$4$};
		\node [style=none] (32) at (7, -5.25) {\large$5$};
		\node [style=none] (33) at (10, -5.25) {\large$6$};
		\node [style=none] (34) at (13, -5.25) {\large$7$};
		\node [style=none] (35) at (16, -5.25) {\large$8$};
		\node [style=none] (36) at (7, 1.25) {\large$1$};
	\end{pgfonlayer}
	\begin{pgfonlayer}{edgelayer}
		\draw [style={dash_2}] (19.center)
			 to (20.center)
			 to (21.center)
			 to (18.center)
			 to cycle;
		\draw [style=thick] (22) to (24);
		\draw [style=thick] (25) to (23);
		\draw [style=thick, bend left=45] (27) to (29);
	\end{pgfonlayer}
\end{tikzpicture}}
		\end{aligned}
	\end{equation}
	corresponding to the set partition
	\begin{equation} \label{setpartexblock}
		\pi = \{1, 3 \mid 2, 4 \mid 5 \mid 7 \mid 6, 8\}
	\end{equation}
	Here, $l = 2$ and $k = 6$. Suppose that $n = 5$. 
	Then the blocks of $x_\pi$ are labelled, in left-to-right order, 
	as $B_1, B_2, B_3, B_5, B_4$.
	Hence, the block labelling for $x_\pi$ is
	\begin{equation}
	(I_{\pi}, J_{\pi}) = (1, 2, 1, 2, 3, 4, 5, 4)
	\end{equation}
	an element of $[5]^{2+6}$, or, in diagram form,
	\begin{equation} \label{blocklabeldiag}
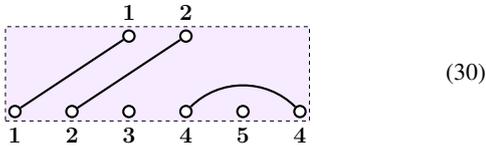
		\begin{aligned}
			\scalebox{0.5}{\begin{tikzpicture}
	\begin{pgfonlayer}{nodelayer}
		\node [style=none] (15) at (7, 1.25) {\huge$\bm{1}$};
		\node [style=none] (16) at (10, 1.25) {\huge$\bm{2}$};
		\node [style=none] (18) at (0.5, 0.5) {};
		\node [style=none] (19) at (0.5, -4.5) {};
		\node [style=none] (20) at (16.5, -4.5) {};
		\node [style=none] (21) at (16.5, 0.5) {};
		\node [style=node] (22) at (1, -4) {};
		\node [style=node] (23) at (10, 0) {};
		\node [style=node] (24) at (7, 0) {};
		\node [style=node] (25) at (4, -4) {};
		\node [style=node] (26) at (7, -4) {};
		\node [style=node] (27) at (10, -4) {};
		\node [style=node] (28) at (13, -4) {};
		\node [style=node] (29) at (16, -4) {};
		\node [style=none] (30) at (1, -5.25) {\huge$\bm{1}$};
		\node [style=none] (31) at (4, -5.25) {\huge$\bm{2}$};
		\node [style=none] (32) at (7, -5.25) {\huge$\bm{3}$};
		\node [style=none] (33) at (10, -5.25) {\huge$\bm{4}$};
		\node [style=none] (34) at (13, -5.25) {\huge$\bm{5}$};
		\node [style=none] (35) at (16, -5.25) {\huge$\bm{4}$};
	\end{pgfonlayer}
	\begin{pgfonlayer}{edgelayer}
		\draw [style={dash_2}] (19.center)
			 to (20.center)
			 to (21.center)
			 to (18.center)
			 to cycle;
		\draw [style=thick] (22) to (24);
		\draw [style=thick] (25) to (23);
		\draw [style=thick, bend left=45] (27) to (29);
	\end{pgfonlayer}
\end{tikzpicture}}
		\end{aligned}
	\end{equation}
	We see how the blocks and their labels have been made clear by using the diagram form
	of the block labelling for $x_\pi$.
\end{example}


The diagram form of the block labelling is very nice for another reason: 
we can easily construct
a matrix unit 
in
$\Hom((\mathbb{R}^{n})^{\otimes k}, (\mathbb{R}^{n})^{\otimes l})$
from it.
This matrix unit is simply
$E_{I_\pi, J_\pi}$, where 
$I_{\pi}$ is the top row of the diagram form of the block labelling
and $J_{\pi}$ is the bottom row of the diagram form of the block labelling.


Moreover, by acting on the block labelling $(I_{\pi}, J_{\pi})$
with $S_n$, we obtain an orbit 
for the $S_n$ action on $[n]^{l+k}$ with $(I_{\pi}, J_{\pi})$ as the class representative.
Denote this orbit by $O((I_{\pi},J_{\pi}))$.
The beauty of the diagram form for the block labelling 
is that it shows explicitly how all of the elements $(I, J)$ in this orbit are precisely
all of the possible labellings of the blocks of $x_\pi$!
Hence, we have that 
\begin{proposition}
$O((I_{\pi},J_{\pi}))$ is equal to
\begin{equation}
	\left\{ (I,J) \in [n]^{l+k} \; \middle\vert 
	\begin{array}{l}
		i_x = i_y \text{ if and only if } x, y \\
		\text{are in the same block of } \pi
	\end{array}
	\right\}
\end{equation}
\end{proposition}

\begin{example}  \label{blocklabelperm}
	Continuing on from Example \ref{blocklabelex},
	the matrix unit that we obtain from (\ref{blocklabeldiag}) is
	$E_{(1,2\mid1,2,3,4,5,4)}$. 
	Moreover,
	we see that 
	\begin{equation} \label{blocklabellingperm}
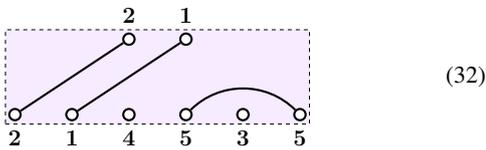
		\begin{aligned}
			\scalebox{0.5}{\begin{tikzpicture}
	\begin{pgfonlayer}{nodelayer}
		\node [style=none] (15) at (7, 1.25) {\huge$\bm{2}$};
		\node [style=none] (16) at (10, 1.25) {\huge$\bm{1}$};
		\node [style=none] (18) at (0.5, 0.5) {};
		\node [style=none] (19) at (0.5, -4.5) {};
		\node [style=none] (20) at (16.5, -4.5) {};
		\node [style=none] (21) at (16.5, 0.5) {};
		\node [style=node] (22) at (1, -4) {};
		\node [style=node] (23) at (10, 0) {};
		\node [style=node] (24) at (7, 0) {};
		\node [style=node] (25) at (4, -4) {};
		\node [style=node] (26) at (7, -4) {};
		\node [style=node] (27) at (10, -4) {};
		\node [style=node] (28) at (13, -4) {};
		\node [style=node] (29) at (16, -4) {};
		\node [style=none] (30) at (1, -5.25) {\huge$\bm{2}$};
		\node [style=none] (31) at (4, -5.25) {\huge$\bm{1}$};
		\node [style=none] (32) at (7, -5.25) {\huge$\bm{4}$};
		\node [style=none] (33) at (10, -5.25) {\huge$\bm{5}$};
		\node [style=none] (34) at (13, -5.25) {\huge$\bm{3}$};
		\node [style=none] (35) at (16, -5.25) {\huge$\bm{5}$};
	\end{pgfonlayer}
	\begin{pgfonlayer}{edgelayer}
		\draw [style={dash_2}] (19.center)
			 to (20.center)
			 to (21.center)
			 to (18.center)
			 to cycle;
		\draw [style=thick] (22) to (24);
		\draw [style=thick] (25) to (23);
		\draw [style=thick, bend left=45] (27) to (29);
	\end{pgfonlayer}
\end{tikzpicture}}
		\end{aligned}
	\end{equation}
	is in the orbit of (\ref{blocklabeldiag}) as a result of
	relabelling the blocks of (\ref{blocklabellingvert}), 
	or, more formally, by applying the permutation
	$(12)(345)$ in $S_5$ to the block labels of (\ref{blocklabeldiag}).
	In particular, we obtain the matrix unit
	$E_{(2,1\mid2,1,4,5,3,5)}$ 
	from (\ref{blocklabellingperm}), 
	which is 
	a linear map from
	$(\mathbb{R}^{5})^{\otimes 6}$
	to
	$(\mathbb{R}^{5})^{\otimes 2}$. 
\end{example}


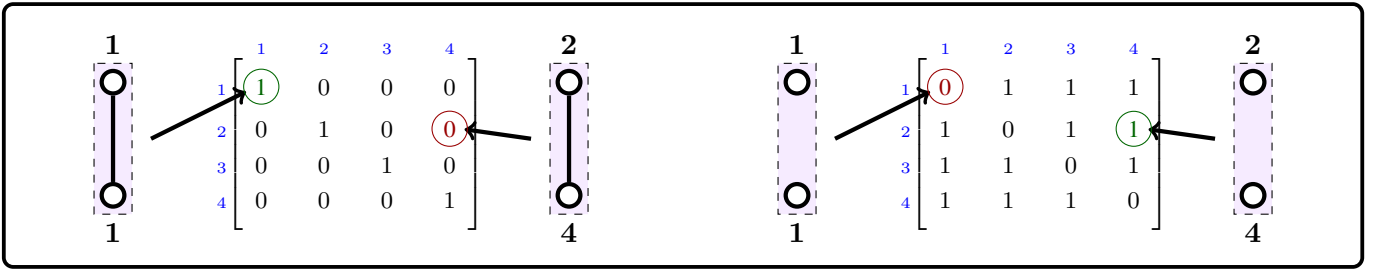
\begin{figure*}[tb]
	\begin{tcolorbox}[colback=white!02, colframe=black]
	\begin{center}
		\begin{tikzpicture}
	\begin{pgfonlayer}{nodelayer}
		\node [style=none] (2) at (3, 1) {\large$\bm{1}$};
		\node [style=node] (3) at (3, 0) {};
		\node [style=node] (4) at (3, -3) {};
		\node [style=none] (5) at (3, -4) {\large$\bm{1}$};
		\node [style=none] (6) at (2.5, 0.5) {};
		\node [style=none] (7) at (2.5, -3.5) {};
		\node [style=none] (8) at (3.5, -3.5) {};
		\node [style=none] (9) at (3.5, 0.5) {};
		\node [style=none] (10) at (9.25, -1.25) {$
	\NiceMatrixOptions{code-for-first-row = \scriptstyle \color{blue},
                   	   code-for-first-col = \scriptstyle \color{blue}
	}
	\begin{bNiceArray}{c *{3}{wc{1.5em}}}[first-row,first-col]
				& 1 & 2& 3	& 4 	\\
		1		& \tikz[baseline=(1-2.base)]{\node[draw,circle,green!40!black,inner sep=2pt] (1-2) {1};} 	& 0	& 0	& 0	\\
		2		& 0	& 1	& 0	& \tikz[baseline=(1-2.base)]{\node[draw,circle,red!60!black,inner sep=2pt] (1-2) {0};}	\\
		3		& 0	& 0	& 1	& 0	\\
		4		& 0	& 0	& 0	& 1
	\end{bNiceArray}
$};
		\node [style=none] (11) at (21, 1) {\large$\bm{1}$};
		\node [style=node] (12) at (21, 0) {};
		\node [style=node] (13) at (21, -3) {};
		\node [style=none] (14) at (21, -4) {\large$\bm{1}$};
		\node [style=none] (15) at (20.5, 0.5) {};
		\node [style=none] (16) at (20.5, -3.5) {};
		\node [style=none] (17) at (21.5, -3.5) {};
		\node [style=none] (18) at (21.5, 0.5) {};
		\node [style=none] (19) at (15, 1) {\large$\bm{2}$};
		\node [style=node] (20) at (15, 0) {};
		\node [style=node] (21) at (15, -3) {};
		\node [style=none] (22) at (15, -4) {\large$\bm{4}$};
		\node [style=none] (23) at (14.5, 0.5) {};
		\node [style=none] (24) at (14.5, -3.5) {};
		\node [style=none] (25) at (15.5, -3.5) {};
		\node [style=none] (26) at (15.5, 0.5) {};
		\node [style=none] (27) at (4, -1.5) {};
		\node [style=none] (28) at (6.5, -0.25) {};
		\node [style=none] (29) at (14, -1.5) {};
		\node [style=none] (30) at (12.25, -1.25) {};
		\node [style=none] (31) at (33, 1) {\large$\bm{2}$};
		\node [style=node] (32) at (33, 0) {};
		\node [style=node] (33) at (33, -3) {};
		\node [style=none] (34) at (33, -4) {\large$\bm{4}$};
		\node [style=none] (35) at (32.5, 0.5) {};
		\node [style=none] (36) at (32.5, -3.5) {};
		\node [style=none] (37) at (33.5, -3.5) {};
		\node [style=none] (38) at (33.5, 0.5) {};
		\node [style=none] (39) at (27.25, -1.25) {$
	\NiceMatrixOptions{code-for-first-row = \scriptstyle \color{blue},
                   	   code-for-first-col = \scriptstyle \color{blue}
	}
	\begin{bNiceArray}{c *{3}{wc{1.5em}}}[first-row,first-col]
				& 1 & 2& 3	& 4 	\\
		1		& \tikz[baseline=(1-2.base)]{\node[draw,circle,red!60!black,inner sep=2pt] (1-2) {0};} 	& 1	& 1	& 1	\\
		2		& 1	& 0	& 1	& \tikz[baseline=(1-2.base)]{\node[draw,circle,green!40!black,inner sep=2pt] (1-2) {1};}	\\
		3		& 1	& 1	& 0	& 1	\\
		4		& 1	& 1	& 1	& 0
	\end{bNiceArray}
$};
		\node [style=none] (40) at (32, -1.5) {};
		\node [style=none] (41) at (30.25, -1.25) {};
		\node [style=none] (42) at (22, -1.5) {};
		\node [style=none] (43) at (24.5, -0.25) {};
	\end{pgfonlayer}
	\begin{pgfonlayer}{edgelayer}
		\draw [style={dash_2}] (6.center)
			 to (7.center)
			 to (8.center)
			 to (9.center)
			 to cycle;
		\draw [style=thick] (4) to (3);
		\draw [style={dash_2}] (17.center)
			 to (18.center)
			 to (15.center)
			 to (16.center)
			 to cycle;
		\draw [style={dash_2}] (23.center)
			 to (24.center)
			 to (25.center)
			 to (26.center)
			 to cycle;
		\draw [style={thick_arrow}] (27.center) to (28.center);
		\draw [style=thick] (21) to (20);
		\draw [style={thick_arrow}] (29.center) to (30.center);
		\draw [style={dash_2}] (37.center)
			 to (38.center)
			 to (35.center)
			 to (36.center)
			 to cycle;
		\draw [style={thick_arrow}] (40.center) to (41.center);
		\draw [style={thick_arrow}] (42.center) to (43.center);
	\end{pgfonlayer}
\end{tikzpicture}
	\end{center}
	\end{tcolorbox}
	\caption{
		We obtain the two basis matrices whose weighted linear combination gives all of the possible
		weight matrices that can appear in an $S_4$-equivariant neural network
		from $\mathbb{R}^4$ to $\mathbb{R}^4$.
		We obtain these matrices from the orbit basis diagrams in $P_1^1(4)$ 
		that have at most $4$ blocks.
		For each orbit basis diagram, to calculate the $(I,J)$-entry of its associated basis matrix,
		we place the $I$-tuple on the top row of the diagram and the $J$-tuple 
		on the bottom row of the diagram and see if they consistently label the diagram's blocks such that no two blocks have the same label.
		If the labelling is consistent, then we put a $1$ in the $(I,J)$-entry of the matrix, otherwise $0$.
		}
  	\label{matrix1,4}
\end{figure*}



The reason for defining the block labelling of an orbit basis diagram $x_\pi$ 
having at most $n$ blocks in $P_k^l(n)$
is that we can use it 
to construct a basis element of 
$\Hom_{S_n}((\mathbb{R}^{n})^{\otimes k}, (\mathbb{R}^{n})^{\otimes l})$ as follows: 
obtaining all of the elements $(I,J)$ that appear in $O((I_{\pi},J_{\pi}))$,
and noting that we can form the matrix unit $E_{I,J}$ from each element,
we can define $X_\pi$ to be
\begin{equation} \label{equivbasiselement}
	X_\pi \coloneqq \sum_{(I,J) \in O((I_{\pi},J_{\pi}))} E_{I,J}
\end{equation}
We see that $X_\pi$ is a basis element of 
$\Hom_{S_n}((\mathbb{R}^{n})^{\otimes k}, (\mathbb{R}^{n})^{\otimes l})$
by (\ref{centeq1}).

Put simply, 
to obtain a basis element $X_\pi$ of
$\Hom_{S_n}((\mathbb{R}^{n})^{\otimes k}, (\mathbb{R}^{n})^{\otimes l})$,
we have added together the matrix units that
come from all of the possible labellings of the blocks of $x_\pi$.

Consequently, we can define the following linear map to make clear
the connection between the partition vector space $P_k^l(n)$ 
and all of the weight matrices that can appear in a permutation equivariant 
neural network between the layers
$(\mathbb{R}^{n})^{\otimes k}$ and $(\mathbb{R}^{n})^{\otimes l}$.

\begin{definition} \label{partmatrixmap}
	For all non-negative integers $l, k$ and positive integers $n$, 
	we can define a surjective map
	\begin{equation} \label{partmatrixmapexp}
		\Phi_{k,n}^l : P_k^l(n) \rightarrow 
		\Hom_{S_n}((\mathbb{R}^{n})^{\otimes k}, (\mathbb{R}^{n})^{\otimes l})
	\end{equation}
	on the orbit basis 
	of $P_k^l(n)$ as follows, and extend linearly:
	\begin{equation} \label{partcentsurj}
		\Phi_{k,n}^l(x_\pi) \coloneqq
		\begin{cases}
			X_\pi & \text{if $\pi$ has $n$ or fewer blocks} \\
			0     & \text{if $\pi$ has more than $n$ blocks}
		\end{cases}
	\end{equation}
\end{definition}
In the case where $k = l$, (\ref{partmatrixmapexp}) is the map that \citet{Jones}
used to obtain Schur--Weyl duality between the symmetric group and the partition algebra.
Hence we have adapted Schur--Weyl duality to characterise the weight matrices
that appear in any permutation equivariant neural network where the layers are some
tensor power of $\mathbb{R}^n$.

We summarise our results with the following two theorems.


\begin{theorem} \label{equivbasis}
	For all non-negative integers $l, k$ and positive integers $n$, 
	we have that
	\begin{equation}
		\{X_\pi \mid \pi \in \Pi_{l+k,n} \}
	\end{equation}
	is a basis of 
	$\Hom_{S_n}((\mathbb{R}^{n})^{\otimes k}, (\mathbb{R}^{n})^{\otimes l})$.
\end{theorem}

\begin{theorem}[Permutation Equivariant Weight Matrices] \label{weightmatclass}
	For all non-negative integers $l, k$ and positive integers $n$, 
	the weight matrix $W$ that appears in
	an $S_n$-equivariant linear layer function
	from $(\mathbb{R}^{n})^{\otimes k}$ to $(\mathbb{R}^{n})^{\otimes l}$
	must be of the form
	\begin{equation}
		W = \sum_{\pi \in \Pi_{l+k,n}} \lambda_\pi{X_\pi}
	\end{equation}
	for $\Bell(l+k,n)$ many weights $\lambda_\pi \in \mathbb{R}$.
\end{theorem}


\subsection{A Note on the Relationship between $n, k$ and $l$}

In classifying the weight matrices that 
can appear in permutation equivariant neural networks,
it is important to note that there is a relationship between $n, k$ and $l$.
In particular, the number of weights in the weight matrix can depend on $n$.

If $n \geq l+k$, we see that the map $\Phi_{k,n}^l$ is an isomorphism of vector spaces.
This is because an orbit basis diagram in $P_k^l(n)$ can have at most $l+k$ blocks,
and so, in this case, there are no orbit basis diagrams with more than $n$ blocks. 
Consequently, the number of weights in the weight matrix does not depend on $n$.

However, if $n < l+k$, then the map $\Phi_{k,n}^l$ is \textit{not} an isomorphism of vector spaces.
Indeed, in this case, the kernel of this map is non-trivial,
of dimension $\Bell(l+k) - \Bell(l+k,n)$,
since it is the $\mathbb{R}$-linear span of the orbit basis diagrams in $P_k^l(n)$
having more than $n$ blocks.
Consequently, the number of weights in the weight matrix \textit{does} depend on $n$.

This improves upon the result that appears in \citet{maron2018}. 
Although this relationship was first mentioned in the Appendix of \citet{finzi}, 
we wish to highlight this point in the main text of our paper
because a number of papers that we have read in the machine learning literature on this topic
assume that the number of weights in the weight matrix is independent of $n$ in all cases.
This becomes more important when $l, k$ are large and $n$ is small, since
the dimension of the kernel becomes very large relative to the actual number of weights
in the weight matrix.
For more information on the kernel of $\Phi_{k,n}^l$, see \citet{BenHal1}.

\subsection{General Procedure and Examples}

In 
Procedure 1,
we provide 
an algorithm
for how to calculate the weight matrix 
that appears in a permutation equivariant neural network
from the layer space
$(\mathbb{R}^{n})^{\otimes k}$
to the layer space
$(\mathbb{R}^{n})^{\otimes l}$
so that our results will be accessible
to the general machine learning practitioner.
In 
Procedure 2,
we describe 
how to calculate the $(I,J)$-entry
of each basis matrix that appears in the overall weight matrix.
This method is powerful because
each $(I,J)$-entry of $X_\pi$ can be calculated simply by placing
the $I$ indices on the top row of the orbit basis diagram $x_\pi$
and the $J$ indices on the bottom row of $x_\pi$ and seeing whether the
blocks of the diagram are consistently and distinctly labelled.

We give a number of examples 
that display the simplicity and power of our method
for calculating any permutation equivariant weight matrix
between tensor power spaces of $\mathbb{R}^{n}$.

\begin{example} \label{114example}
	Suppose that we would like to find 
	the weight matrix for an $S_4$-equivariant linear layer function from
	$\mathbb{R}^{4}$ to $\mathbb{R}^{4}$.
	Note that $l = k = 1$ and $n = 4$.

	To calculate this weight matrix, we follow Procedure 1.
	First, we need to calculate all of the set partitions of $[1+1]$ having at most $4$ blocks.
	These are 
	$\pi_1 = \{1, 2\}$ and 
	$\pi_2 = \{1 \mid 2\}$.
	Next, we express each of these set partitions as an orbit basis diagram in $P_1^1(4)$.
	These diagrams appeared in Example \ref{exP11n}.
	Now we apply the map $\Phi_{1,4}^{1}$ to each of these orbit basis diagrams
	to obtain the basis matrices 
	$X_{\pi_1}$ 
	and 
	$X_{\pi_2}$.
	Figure \ref{matrix1,4}
	shows how to calculate all of the $(I,J)$-entries for both of these matrices using 
	Procedure 2.
	In particular, we see, for example, that the $(1,1)$-entry of $X_{\pi_1}$ is $1$, since the only block in $x_{\pi_1}$ is consistently labelled (by $1$), whereas the $(2,4)$-entry of $X_{\pi_1}$ is $0$, since the only block in $x_{\pi_1}$ is inconsistently labelled ($2 \neq 4$). The procedure is the same for $X_{\pi_2}$. We see that only the diagonal entries of $X_{\pi_2}$ are zero since the two blocks in $x_{\pi_2}$ must be distinctly labelled.

	Finally, we multiply each matrix by a weight, namely $\lambda_1$ and $\lambda_2$, respectively, and then
	add the two matrices together to obtain the overall weight matrix.
	Hence, the weight matrix for an $S_4$-equivariant linear layer function from
$\mathbb{R}^{4}$ to 
$\mathbb{R}^{4}$
is of the form
\begin{equation}
	\NiceMatrixOptions{code-for-first-row = \scriptstyle \color{blue},
                   	   code-for-first-col = \scriptstyle \color{blue}
	}
	\renewcommand{\arraystretch}{1.5}
	\begin{bNiceArray}{*{4}{c}}[first-row,first-col]
				& 1 		& 2	& 3	& 4 \\
		1		& \lambda_1	& \lambda_2	& \lambda_2	& \lambda_2	\\
		2		& \lambda_2	& \lambda_1	& \lambda_2	& \lambda_2	\\
		3		& \lambda_2	& \lambda_2	& \lambda_1	& \lambda_2	\\
		4		& \lambda_2	& \lambda_2	& \lambda_2	& \lambda_1
	\end{bNiceArray}
\end{equation}
for weights $\lambda_1, \lambda_2 \in \mathbb{R}$.


It is not hard to see that the weight matrix for an $S_n$-equivariant linear layer function from
$\mathbb{R}^{n}$ to 
$\mathbb{R}^{n}$
	is an $n \times n$ matrix, with the diagonal entries given by the weight $\lambda_1$ and the off-diagonal entries given by the weight 
$\lambda_2$.
\end{example}


\begin{figure*}[tb]
	\begin{tcolorbox}[colback=white!02, colframe=black]
	\begin{center}
		\scalebox{0.6}{\begin{tikzpicture}
	\begin{pgfonlayer}{nodelayer}
		\node [style=none] (0) at (0.5, -4) {$3$};
		\node [style=node] (1) at (0.5, 0) {};
		\node [style=node] (2) at (0.5, -3) {};
		\node [style=node] (3) at (3.5, 0) {};
		\node [style=node] (4) at (3.5, -3) {};
		\node [style=none] (5) at (3.5, -4) {$4$};
		\node [style=none] (6) at (0.5, 1) {$1$};
		\node [style=none] (7) at (3.5, 1) {$2$};
		\node [style=none] (8) at (7.5, -4) {$3$};
		\node [style=node] (9) at (7.5, 0) {};
		\node [style=node] (10) at (7.5, -3) {};
		\node [style=node] (11) at (10.5, 0) {};
		\node [style=node] (12) at (10.5, -3) {};
		\node [style=none] (13) at (10.5, -4) {$4$};
		\node [style=none] (14) at (7.5, 1) {$1$};
		\node [style=none] (15) at (10.5, 1) {$2$};
		\node [style=none] (16) at (14.5, -4) {$3$};
		\node [style=node] (17) at (14.5, 0) {};
		\node [style=node] (18) at (14.5, -3) {};
		\node [style=node] (19) at (17.5, 0) {};
		\node [style=node] (20) at (17.5, -3) {};
		\node [style=none] (21) at (17.5, -4) {$4$};
		\node [style=none] (22) at (14.5, 1) {$1$};
		\node [style=none] (23) at (17.5, 1) {$2$};
		\node [style=none] (24) at (21.5, -4) {$3$};
		\node [style=node] (25) at (21.5, 0) {};
		\node [style=node] (26) at (21.5, -3) {};
		\node [style=node] (27) at (24.5, 0) {};
		\node [style=node] (28) at (24.5, -3) {};
		\node [style=none] (29) at (24.5, -4) {$4$};
		\node [style=none] (30) at (21.5, 1) {$1$};
		\node [style=none] (31) at (24.5, 1) {$2$};
		\node [style=none] (32) at (28.5, -4) {$3$};
		\node [style=node] (33) at (28.5, 0) {};
		\node [style=node] (34) at (28.5, -3) {};
		\node [style=node] (35) at (31.5, 0) {};
		\node [style=node] (36) at (31.5, -3) {};
		\node [style=none] (37) at (31.5, -4) {$4$};
		\node [style=none] (38) at (28.5, 1) {$1$};
		\node [style=none] (39) at (31.5, 1) {$2$};
		\node [style=none] (40) at (49.5, -4) {$3$};
		\node [style=node] (41) at (49.5, 0) {};
		\node [style=node] (42) at (49.5, -3) {};
		\node [style=node] (43) at (52.5, 0) {};
		\node [style=node] (44) at (52.5, -3) {};
		\node [style=none] (45) at (52.5, -4) {$4$};
		\node [style=none] (46) at (49.5, 1) {$1$};
		\node [style=none] (47) at (52.5, 1) {$2$};
		\node [style=none] (48) at (35.5, -4) {$3$};
		\node [style=node] (49) at (35.5, 0) {};
		\node [style=node] (50) at (35.5, -3) {};
		\node [style=node] (51) at (38.5, 0) {};
		\node [style=node] (52) at (38.5, -3) {};
		\node [style=none] (53) at (38.5, -4) {$4$};
		\node [style=none] (54) at (35.5, 1) {$1$};
		\node [style=none] (55) at (38.5, 1) {$2$};
		\node [style=none] (56) at (42.5, -4) {$3$};
		\node [style=node] (57) at (42.5, 0) {};
		\node [style=node] (58) at (42.5, -3) {};
		\node [style=node] (59) at (45.5, 0) {};
		\node [style=node] (60) at (45.5, -3) {};
		\node [style=none] (61) at (45.5, -4) {$4$};
		\node [style=none] (62) at (42.5, 1) {$1$};
		\node [style=none] (63) at (45.5, 1) {$2$};
		\node [style=none] (64) at (49, 0.5) {};
		\node [style=none] (65) at (49, -3.5) {};
		\node [style=none] (66) at (53, -3.5) {};
		\node [style=none] (67) at (53, 0.5) {};
		\node [style=none] (68) at (42, 0.5) {};
		\node [style=none] (69) at (42, -3.5) {};
		\node [style=none] (70) at (46, -3.5) {};
		\node [style=none] (71) at (46, 0.5) {};
		\node [style=none] (72) at (35, 0.5) {};
		\node [style=none] (73) at (35, -3.5) {};
		\node [style=none] (74) at (39, -3.5) {};
		\node [style=none] (75) at (39, 0.5) {};
		\node [style=none] (76) at (14, 0.5) {};
		\node [style=none] (77) at (14, -3.5) {};
		\node [style=none] (78) at (18, -3.5) {};
		\node [style=none] (79) at (18, 0.5) {};
		\node [style=none] (80) at (7, 0.5) {};
		\node [style=none] (81) at (7, -3.5) {};
		\node [style=none] (82) at (11, -3.5) {};
		\node [style=none] (83) at (11, 0.5) {};
		\node [style=none] (84) at (28, 0.5) {};
		\node [style=none] (85) at (28, -3.5) {};
		\node [style=none] (86) at (32, -3.5) {};
		\node [style=none] (87) at (32, 0.5) {};
		\node [style=none] (88) at (21, 0.5) {};
		\node [style=none] (89) at (21, -3.5) {};
		\node [style=none] (90) at (25, -3.5) {};
		\node [style=none] (91) at (25, 0.5) {};
		\node [style=none] (92) at (0, 0.5) {};
		\node [style=none] (93) at (0, -3.5) {};
		\node [style=none] (94) at (4, -3.5) {};
		\node [style=none] (95) at (4, 0.5) {};
	\end{pgfonlayer}
	\begin{pgfonlayer}{edgelayer}
		\draw [style={dash_2}] (64.center)
			 to (65.center)
			 to (66.center)
			 to (67.center)
			 to cycle;
		\draw [style={dash_2}] (68.center)
			 to (69.center)
			 to (70.center)
			 to (71.center)
			 to cycle;
		\draw [style={dash_2}] (72.center)
			 to (73.center)
			 to (74.center)
			 to (75.center)
			 to cycle;
		\draw [style={dash_2}] (76.center)
			 to (77.center)
			 to (78.center)
			 to (79.center)
			 to cycle;
		\draw [style={dash_2}] (80.center)
			 to (81.center)
			 to (82.center)
			 to (83.center)
			 to cycle;
		\draw [style={dash_2}] (84.center)
			 to (85.center)
			 to (86.center)
			 to (87.center)
			 to cycle;
		\draw [style={dash_2}] (88.center)
			 to (89.center)
			 to (90.center)
			 to (91.center)
			 to cycle;
		\draw [style={dash_2}] (92.center)
			 to (93.center)
			 to (94.center)
			 to (95.center)
			 to cycle;
		\draw [style=thick] (2) to (1);
		\draw [style=thick] (2) to (4);
		\draw [style=thick] (4) to (3);
		\draw [style=thick] (3) to (1);
		\draw [style=thick] (2) to (3);
		\draw [style=thick] (4) to (1);
		\draw [style=thick] (28) to (27);
		\draw [style=thick] (27) to (25);
		\draw [style=thick] (28) to (25);
		\draw [style=thick] (34) to (35);
		\draw [style=thick] (35) to (33);
		\draw [style=thick] (34) to (33);
		\draw [style=thick] (10) to (11);
		\draw [style=thick] (11) to (12);
		\draw [style=thick] (12) to (10);
		\draw [style=thick] (17) to (18);
		\draw [style=thick] (18) to (20);
		\draw [style=thick] (20) to (17);
		\draw [style=thick] (49) to (51);
		\draw [style=thick] (50) to (52);
		\draw [style=thick] (57) to (60);
		\draw [style=thick] (58) to (59);
		\draw [style=thick] (41) to (42);
		\draw [style=thick] (44) to (43);
	\end{pgfonlayer}
\end{tikzpicture}}
	\end{center}
	\end{tcolorbox}
	\caption{We show the eight orbit basis diagrams in $P_2^2(2)$ that have at most $2$ blocks.
	They are needed to calculate the weight matrix
		for an $S_2$-equivariant linear layer function 
		$(\mathbb{R}^{2})^{\otimes 2} \rightarrow (\mathbb{R}^{2})^{\otimes 2}$.
		As the number of orbit basis diagrams in $P_2^2(2)$ is $\Bell(4) = 15$, 
		this example highlights that the number of weights that appear in a permutation equivariant
		weight matrix depends on the relationship between the degree $n$ of the symmetric group $S_n$ and the sum of the tensor power orders $l + k$ that define the layers of the permutation equivariant neural network.
		}
  	\label{matrix2,2}
\end{figure*}
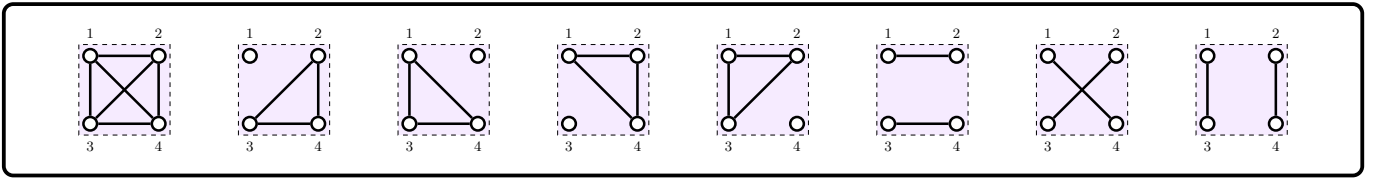


\begin{example}
	Continuing on from Example \ref{blocklabelex},
	we see that 
	the $(1, 2 \mid 1, 2, 3, 4, 5, 4)$-entry of the weight matrix 
	for an $S_5$-equivariant linear layer function
	from
	$(\mathbb{R}^{5})^{\otimes 6}$
	to 
	$(\mathbb{R}^{5})^{\otimes 2}$
	will be $\lambda_{\pi}$, 
	a parameter that corresponds to the set partition $\pi$ given in (\ref{setpartexblock}).

	This is because the diagram given in (\ref{blocklabeldiag}),
	where the orbit basis diagram corresponding to $\pi$ has had the top row 
	overlaid with the indices of $I = (1, 2)$ and the bottom row with 
	$J = (1, 2, 3, 4, 5, 4)$, satisfies condition 2 of 
	Procedure 2,
	namely
	that the indices consistently and distinctly label the blocks of $x_\pi$.

	Moreover, referring back to Example \ref{blocklabelperm}, we see that
	the $(2, 1 \mid 2, 1, 4, 5, 3, 5)$-entry of the same weight matrix 
	will also be $\lambda_{\pi}$, since this is related to $(1, 2 \mid 1, 2, 3, 4, 5, 4)$
	by the permutation $(1 2)(3 4 5)$ in $S_5$.
\end{example}


\begin{example} \label{lessBell}
	We now give an example where the number of weights
	in a permutation equivariant weight matrix
	is not the full Bell number $\Bell(l+k)$.
	Suppose that we would like to find the weight matrix for 
	an $S_2$-equivariant linear layer function from 
	$(\mathbb{R}^{2})^{\otimes 2}$
	to 
	$(\mathbb{R}^{2})^{\otimes 2}$.
	In this case, $l = k = 2$ and $n = 2$.

	To calculate this weight matrix, we again follow Procedure 1.
	We first need to calculate all of the set partitions of $[2+2]$ having at most $2$ blocks.
	There are $\Bell(4,2) = 8$ of them, and they are shown in Figure \ref{matrix2,2}.
	Note, in particular, that there are not $\Bell(4) = 15$ of them, 
	which implies that the map $\Phi_{2,2}^2$ has a kernel. 
	This is what we expected, since $n \ngeq l+k$.

	Next, we apply the map $\Phi_{2,2}^2$ to each of the eight 
	orbit basis diagrams to obtain eight basis matrices $X_{\pi_1}, \dots, X_{\pi_8}$, multiply each matrix $X_{\pi_i}$ by a weight $\lambda_i$,
	and then finally add them all together.


Hence, the weight matrix for an $S_2$-equivariant linear layer function from
$(\mathbb{R}^{2})^{\otimes 2}$ to 
$(\mathbb{R}^{2})^{\otimes 2}$
is of the form

\begin{equation}
	\NiceMatrixOptions{code-for-first-row = \scriptstyle \color{blue},
                   	   code-for-first-col = \scriptstyle \color{blue}
	}
	\renewcommand{\arraystretch}{1.5}
	\begin{bNiceArray}{*{2}{c}|*{2}{c}}[first-row,first-col]
				& 1,1 		& 1,2	& 2,1	& 2,2 \\
		1,1		& \lambda_1	& \lambda_3	& \lambda_2	& \lambda_6	\\
		1,2		& \lambda_5	& \lambda_8	& \lambda_7	& \lambda_4	\\
		\hline
		2,1		& \lambda_4	& \lambda_7	& \lambda_8	& \lambda_5	\\
		2,2		& \lambda_6	& \lambda_2	& \lambda_3	& \lambda_1
	\end{bNiceArray}
\end{equation}
for weights $\lambda_1, \lambda_2, \dots, \lambda_8 \in \mathbb{R}$.
\end{example}

\subsection{Adding Features and Biases} \label{featbiases}

Adding features and biases was first considered by \citet{maron2018}; in the 
Supplementary Material \cite{supp}
we show how the basis matrices with features and biases can be found in terms of orbit basis diagrams 
by adapting the results that appear in Section \ref{permequivmat}.

\subsection{Equivariance to Local Symmetries}
We can extend our results to linear layer functions 
that are equivariant to a direct product of symmetric groups
$S_{n_1} \times \dots \times S_{n_m}$.
These functions model local symmetries in data since each symmetric group $S_{n_r}$
in the direct product captures only the symmetries in its associated subset of $n_r$ objects.
We 
can
use our method to recover the result of \citet{Hartford} and
give an explanation 
in the language of the partition algebras as to why their result holds.
These extensions are discussed 
in the Supplementary Material \cite{supp}.

\subsection{Limitations and Discussion}

It is important to acknowledge that given the
current limitations of hardware, there will be some challenges when
implementing the neural networks that are discussed in this paper.
In particular, significant engineering efforts will be needed to achieve the required scale because storing high-order tensors in memory
is not a straightforward task. 
This was demonstrated by \citet{clebschgordan}, who had to develop custom CUDA kernels in order to implement their tensor product based neural networks. 
Nevertheless, we anticipate that with the increasing availability of computing power, higher-order group equivariant neural networks will become more prevalent in practical applications. 
Notably, while the dimension of tensor power spaces increases exponentially with their order, the dimension of the space of equivariant maps between such tensor power spaces is often much smaller, and the corresponding matrices are typically sparse. 
Therefore, while storing these matrices may present some technical difficulties, it should be feasible with the current computing power that is available.



\subsection{Code}

Together with 
Procedures 1 and 2,
we have provided
a PyTorch implementation of the permutation equivariant weight matrices
for any symmetric group $S_n$ and for all possible tensor power spaces
of $\mathbb{R}^n$ 
in the Supplementary Material \cite{supp}.
This will make it possible for the general machine learning practitioner
to use the layers that we have characterised
in their experiments.




\section{Conclusion} \label{conclusion}

We are the first to show that Schur--Weyl duality 
between the symmetric group and the partition algebra
can be used to fully characterise the permutation equivariant weight matrices 
that appear between neural network layers that are tensor power spaces of $\mathbb{R}^{n}$.
We showed that the weight matrices can be obtained by constructing a basis of matrices
from a vector space of diagrams that is adapted from the partition algebra.
In particular, we proved that each basis matrix can be found from its associated orbit 
basis diagram by
adding together 
all of the matrix units that are indexed by all of the possible labellings of the blocks in the diagram.
In doing so, we have added weight to the idea that 
Schur--Weyl duality is a useful
tool for constructing
group equivariant neural network architectures.



\begin{ack}
	The author would like to thank his PhD supervisor Professor William J. Knottenbelt for being generous with his time throughout the author's period of research prior to the publication of this paper.
	This work was funded by the Doctoral Scholarship for Applied Research which was awarded to the author under Imperial College London's Department of Computing Applied Research scheme.
	This work will form part of the author's PhD thesis at Imperial College London.
\end{ack}




\nocite{*}
\bibliography{mybibfile}

\newpage
\appendix
\onecolumn

\begin{center}
	{\LARGE\bfseries Supplementary Material}
\end{center}


	\section{Basis Matrix Calculations for the Weight Matrix in Example \ref{lessBell}}
	\label{calcnotbell2k}

	In Example \ref{lessBell}, we gave an example where the number of weights in a permutation
	equivariant weight matrix is not the full Bell number $\Bell(l+k)$.
	In particular, we found that the weight matrix for an $S_2$-equivariant linear layer
	function from 
	$(\mathbb{R}^2)^{\otimes 2}$
	to	
	$(\mathbb{R}^2)^{\otimes 2}$
	had $\Bell(4,2) = 8$ weights, not $\Bell(4) = 15$ weights.
	The following table shows how to calculate the eight basis matrices 
	that determine the weight matrix from the orbit basis elements 
	in $P_2^2(2)$ that have at most $2$ blocks.

	\begin{center}
\begin{tblr}{
  colspec = {X[c,h]X[c]X[c]X[c]},
  stretch = 0,
  rowsep = 5pt,
  hlines = {1pt},
  vlines = {1pt},
}
	{Orbit Basis Element \\ $x_\pi$} 	& {Set Partition \\ $\pi$} 	& {Block Labelling \\ $(I_{\pi} \mid J_{\pi})$}	& 
	{Basis Matrix \\ $X_\pi$} \\
	\scalebox{0.6}{\begin{tikzpicture}
	\begin{pgfonlayer}{nodelayer}
		\node [style=none] (0) at (0, 1) {$1$};
		\node [style=node] (1) at (0, 0) {};
		\node [style=node] (2) at (0, -3) {};
		\node [style=node] (3) at (3, 0) {};
		\node [style=node] (4) at (3, -3) {};
		\node [style=none] (5) at (3, 1) {$2$};
		\node [style=none] (6) at (0, -4) {$3$};
		\node [style=none] (7) at (3, -4) {$4$};
		\node [style=none] (8) at (-0.5, 0.5) {};
		\node [style=none] (9) at (-0.5, -3.5) {};
		\node [style=none] (10) at (3.5, -3.5) {};
		\node [style=none] (11) at (3.5, 0.5) {};
	\end{pgfonlayer}
	\begin{pgfonlayer}{edgelayer}
		\draw [style={dash_2}] (8.center)
			 to (9.center)
			 to (10.center)
			 to (11.center)
			 to cycle;
		\draw [style=thick] (1) to (4);
		\draw [style=thick] (2) to (3);
		\draw [style=thick] (2) to (1);
		\draw [style=thick] (2) to (4);
		\draw [style=thick] (4) to (3);
		\draw [style=thick] (3) to (1);
	\end{pgfonlayer}
\end{tikzpicture}} & $\{1, 2, 3, 4\}$ & $(1, 1 \mid 1, 1)$ & 
	\scalebox{0.75}{
	$
	\NiceMatrixOptions{code-for-first-row = \scriptstyle \color{blue},
                   	   code-for-first-col = \scriptstyle \color{blue}
	}
	\begin{bNiceArray}{*{2}{c} | *{2}{c}}[first-row,first-col]
				& 1,1 		& 1,2	& 2,1	& 2,2 \\
		1,1		& 1	& 0	& 0	& 0	\\
		1,2		& 0	& 0	& 0	& 0	\\
		\cline{1-4}
		2,1		& 0	& 0	& 0	& 0	\\
		2,2		& 0	& 0	& 0	& 1
	\end{bNiceArray}
	$}
	\\
	\scalebox{0.6}{\begin{tikzpicture}
	\begin{pgfonlayer}{nodelayer}
		\node [style=none] (0) at (0, 1) {$1$};
		\node [style=node] (1) at (0, 0) {};
		\node [style=node] (2) at (0, -3) {};
		\node [style=node] (3) at (3, 0) {};
		\node [style=node] (4) at (3, -3) {};
		\node [style=none] (5) at (3, 1) {$2$};
		\node [style=none] (6) at (0, -4) {$3$};
		\node [style=none] (7) at (3, -4) {$4$};
		\node [style=none] (8) at (-0.5, 0.5) {};
		\node [style=none] (9) at (-0.5, -3.5) {};
		\node [style=none] (10) at (3.5, -3.5) {};
		\node [style=none] (11) at (3.5, 0.5) {};
	\end{pgfonlayer}
	\begin{pgfonlayer}{edgelayer}
		\draw [style={dash_2}] (8.center)
			 to (9.center)
			 to (10.center)
			 to (11.center)
			 to cycle;
		\draw [style=thick] (4) to (3);
		\draw [style=thick] (3) to (1);
		\draw [style=thick] (4) to (1);
	\end{pgfonlayer}
\end{tikzpicture}}	& $\{1, 2, 4 \mid 3\}$ 	& $(1, 1 \mid 2, 1)$ 	& 
	\scalebox{0.75}{
	$
	\NiceMatrixOptions{code-for-first-row = \scriptstyle \color{blue},
                   	   code-for-first-col = \scriptstyle \color{blue}
	}
	\begin{bNiceArray}{*{2}{c} | *{2}{c}}[first-row,first-col]
				& 1,1 		& 1,2	& 2,1	& 2,2 \\
		1,1		& 0	& 0	& 1	& 0	\\
		1,2		& 0	& 0	& 0	& 0	\\
		\cline{1-4}
		2,1		& 0	& 0	& 0	& 0	\\
		2,2		& 0	& 1	& 0	& 0
	\end{bNiceArray}
	$}
	\\
	\scalebox{0.6}{\begin{tikzpicture}
	\begin{pgfonlayer}{nodelayer}
		\node [style=none] (0) at (0, 1) {$1$};
		\node [style=node] (1) at (0, 0) {};
		\node [style=node] (2) at (0, -3) {};
		\node [style=node] (3) at (3, 0) {};
		\node [style=node] (4) at (3, -3) {};
		\node [style=none] (5) at (3, 1) {$2$};
		\node [style=none] (6) at (0, -4) {$3$};
		\node [style=none] (7) at (3, -4) {$4$};
		\node [style=none] (8) at (-0.5, 0.5) {};
		\node [style=none] (9) at (-0.5, -3.5) {};
		\node [style=none] (10) at (3.5, -3.5) {};
		\node [style=none] (11) at (3.5, 0.5) {};
	\end{pgfonlayer}
	\begin{pgfonlayer}{edgelayer}
		\draw [style={dash_2}] (8.center)
			 to (9.center)
			 to (10.center)
			 to (11.center)
			 to cycle;
		\draw [style=thick] (2) to (1);
		\draw [style=thick] (1) to (3);
		\draw [style=thick] (2) to (3);
	\end{pgfonlayer}
\end{tikzpicture}}	& $\{1, 2, 3 \mid 4\}$ 	& $(1, 1 \mid 1, 2)$ 	& 
	\scalebox{0.75}{
	$
	\NiceMatrixOptions{code-for-first-row = \scriptstyle \color{blue},
                   	   code-for-first-col = \scriptstyle \color{blue}
	}
	\begin{bNiceArray}{*{2}{c} | *{2}{c}}[first-row,first-col]
				& 1,1 		& 1,2	& 2,1	& 2,2 \\
		1,1		& 0	& 1	& 0	& 0	\\
		1,2		& 0	& 0	& 0	& 0	\\
		\cline{1-4}
		2,1		& 0	& 0	& 0	& 0	\\
		2,2		& 0	& 0	& 1	& 0
	\end{bNiceArray}
	$}
	\\
	\scalebox{0.6}{\begin{tikzpicture}
	\begin{pgfonlayer}{nodelayer}
		\node [style=none] (0) at (0.25, 1) {$1$};
		\node [style=node] (1) at (0.25, 0) {};
		\node [style=node] (2) at (0.25, -3) {};
		\node [style=node] (3) at (3.25, 0) {};
		\node [style=node] (4) at (3.25, -3) {};
		\node [style=none] (5) at (3.25, 1) {$2$};
		\node [style=none] (6) at (0.25, -4) {$3$};
		\node [style=none] (7) at (3.25, -4) {$4$};
		\node [style=none] (8) at (-0.25, 0.5) {};
		\node [style=none] (9) at (-0.25, -3.5) {};
		\node [style=none] (10) at (3.75, -3.5) {};
		\node [style=none] (11) at (3.75, 0.5) {};
	\end{pgfonlayer}
	\begin{pgfonlayer}{edgelayer}
		\draw [style={dash_2}] (8.center)
			 to (9.center)
			 to (10.center)
			 to (11.center)
			 to cycle;
		\draw [style=thick] (2) to (4);
		\draw [style=thick] (4) to (3);
		\draw [style=thick] (2) to (3);
	\end{pgfonlayer}
\end{tikzpicture}}	& $\{1 \mid 2, 3, 4\}$ 	& $(1, 2 \mid 2, 2)$ 	& 
	\scalebox{0.75}{
	$
	\NiceMatrixOptions{code-for-first-row = \scriptstyle \color{blue},
                   	   code-for-first-col = \scriptstyle \color{blue}
	}
	\begin{bNiceArray}{*{2}{c} | *{2}{c}}[first-row,first-col]
				& 1,1 		& 1,2	& 2,1	& 2,2 \\
		1,1		& 0	& 0	& 0	& 0	\\
		1,2		& 0	& 0	& 0	& 1	\\
		\cline{1-4}
		2,1		& 1	& 0	& 0	& 0	\\
		2,2		& 0	& 0	& 0	& 0
	\end{bNiceArray}
	$}
	\\
	\scalebox{0.6}{\begin{tikzpicture}
	\begin{pgfonlayer}{nodelayer}
		\node [style=none] (0) at (0, 1) {$1$};
		\node [style=node] (1) at (0, 0) {};
		\node [style=node] (2) at (0, -3) {};
		\node [style=node] (3) at (3, 0) {};
		\node [style=node] (4) at (3, -3) {};
		\node [style=none] (5) at (3, 1) {$2$};
		\node [style=none] (6) at (0, -4) {$3$};
		\node [style=none] (7) at (3, -4) {$4$};
		\node [style=none] (8) at (-0.5, 0.5) {};
		\node [style=none] (9) at (-0.5, -3.5) {};
		\node [style=none] (10) at (3.5, -3.5) {};
		\node [style=none] (11) at (3.5, 0.5) {};
	\end{pgfonlayer}
	\begin{pgfonlayer}{edgelayer}
		\draw [style={dash_2}] (8.center)
			 to (9.center)
			 to (10.center)
			 to (11.center)
			 to cycle;
		\draw [style=thick] (1) to (2);
		\draw [style=thick] (2) to (4);
		\draw [style=thick] (4) to (1);
	\end{pgfonlayer}
\end{tikzpicture}}	& $\{1, 3, 4 \mid 2\}$ 	& $(1, 2 \mid 1, 1)$ 	& 
	\scalebox{0.75}{
	$
	\NiceMatrixOptions{code-for-first-row = \scriptstyle \color{blue},
                   	   code-for-first-col = \scriptstyle \color{blue}
	}
	\begin{bNiceArray}{*{2}{c} | *{2}{c}}[first-row,first-col]
				& 1,1 		& 1,2	& 2,1	& 2,2 \\
		1,1		& 0	& 0	& 0	& 0	\\
		1,2		& 1	& 0	& 0	& 0	\\
		\cline{1-4}
		2,1		& 0	& 0	& 0	& 1	\\
		2,2		& 0	& 0	& 0	& 0
	\end{bNiceArray}
	$}
	\\
	\scalebox{0.6}{\begin{tikzpicture}
	\begin{pgfonlayer}{nodelayer}
		\node [style=none] (0) at (0, 1) {$1$};
		\node [style=node] (1) at (0, 0) {};
		\node [style=node] (2) at (0, -3) {};
		\node [style=node] (3) at (3, 0) {};
		\node [style=node] (4) at (3, -3) {};
		\node [style=none] (5) at (3, 1) {$2$};
		\node [style=none] (6) at (0, -4) {$3$};
		\node [style=none] (7) at (3, -4) {$4$};
		\node [style=none] (8) at (-0.5, 0.5) {};
		\node [style=none] (9) at (-0.5, -3.5) {};
		\node [style=none] (10) at (3.5, -3.5) {};
		\node [style=none] (11) at (3.5, 0.5) {};
	\end{pgfonlayer}
	\begin{pgfonlayer}{edgelayer}
		\draw [style={dash_2}] (8.center)
			 to (9.center)
			 to (10.center)
			 to (11.center)
			 to cycle;
		\draw [style=thick] (2) to (4);
		\draw [style=thick] (1) to (3);
	\end{pgfonlayer}
\end{tikzpicture}}	& $\{1, 2 \mid 3, 4\}$ 	& $(1, 1 \mid 2, 2)$ 	& 
	\scalebox{0.75}{
	$
	\NiceMatrixOptions{code-for-first-row = \scriptstyle \color{blue},
                   	   code-for-first-col = \scriptstyle \color{blue}
	}
	\begin{bNiceArray}{*{2}{c} | *{2}{c}}[first-row,first-col]
				& 1,1 		& 1,2	& 2,1	& 2,2 \\
		1,1		& 0	& 0	& 0	& 1	\\
		1,2		& 0	& 0	& 0	& 0	\\
		\cline{1-4}
		2,1		& 0	& 0	& 0	& 0	\\
		2,2		& 1	& 0	& 0	& 0
	\end{bNiceArray}
	$}
	\\
	\scalebox{0.6}{\begin{tikzpicture}
	\begin{pgfonlayer}{nodelayer}
		\node [style=none] (0) at (0, 1) {$1$};
		\node [style=node] (1) at (0, 0) {};
		\node [style=node] (2) at (0, -3) {};
		\node [style=node] (3) at (3, 0) {};
		\node [style=node] (4) at (3, -3) {};
		\node [style=none] (5) at (3, 1) {$2$};
		\node [style=none] (6) at (0, -4) {$3$};
		\node [style=none] (7) at (3, -4) {$4$};
		\node [style=none] (8) at (-0.5, 0.5) {};
		\node [style=none] (9) at (-0.5, -3.5) {};
		\node [style=none] (10) at (3.5, -3.5) {};
		\node [style=none] (11) at (3.5, 0.5) {};
	\end{pgfonlayer}
	\begin{pgfonlayer}{edgelayer}
		\draw [style={dash_2}] (8.center)
			 to (9.center)
			 to (10.center)
			 to (11.center)
			 to cycle;
		\draw [style=thick] (1) to (4);
		\draw [style=thick] (2) to (3);
	\end{pgfonlayer}
\end{tikzpicture}}	& $\{1, 4 \mid 2, 3\}$ 	& $(1, 2 \mid 2, 1)$ 	& 
	\scalebox{0.75}{
	$
	\NiceMatrixOptions{code-for-first-row = \scriptstyle \color{blue},
                   	   code-for-first-col = \scriptstyle \color{blue}
	}
	\begin{bNiceArray}{*{2}{c} | *{2}{c}}[first-row,first-col]
				& 1,1 		& 1,2	& 2,1	& 2,2 \\
		1,1		& 0	& 0	& 0	& 0	\\
		1,2		& 0	& 0	& 1	& 0	\\
		\cline{1-4}
		2,1		& 0	& 1	& 0	& 0	\\
		2,2		& 0	& 0	& 0	& 0
	\end{bNiceArray}
	$}
	\\
	\scalebox{0.6}{\begin{tikzpicture}
	\begin{pgfonlayer}{nodelayer}
		\node [style=none] (0) at (0, 1) {$1$};
		\node [style=node] (1) at (0, 0) {};
		\node [style=node] (2) at (0, -3) {};
		\node [style=node] (3) at (3, 0) {};
		\node [style=node] (4) at (3, -3) {};
		\node [style=none] (5) at (3, 1) {$2$};
		\node [style=none] (6) at (0, -4) {$3$};
		\node [style=none] (7) at (3, -4) {$4$};
		\node [style=none] (8) at (-0.5, 0.5) {};
		\node [style=none] (9) at (-0.5, -3.5) {};
		\node [style=none] (10) at (3.5, -3.5) {};
		\node [style=none] (11) at (3.5, 0.5) {};
	\end{pgfonlayer}
	\begin{pgfonlayer}{edgelayer}
		\draw [style={dash_2}] (11.center)
			 to (8.center)
			 to (9.center)
			 to (10.center)
			 to cycle;
		\draw [style=thick] (2) to (1);
		\draw [style=thick] (4) to (3);
	\end{pgfonlayer}
\end{tikzpicture}}	& $\{1, 3 \mid 2, 4\}$ 	& $(1, 2 \mid 1, 2)$ 	& 
	\scalebox{0.75}{
	$
	\NiceMatrixOptions{code-for-first-row = \scriptstyle \color{blue},
                   	   code-for-first-col = \scriptstyle \color{blue}
	}
	\begin{bNiceArray}{*{2}{c} | *{2}{c}}[first-row,first-col]
				& 1,1 		& 1,2	& 2,1	& 2,2 \\
		1,1		& 0	& 0	& 0	& 0	\\
		1,2		& 0	& 1	& 0	& 0	\\
		\cline{1-4}
		2,1		& 0	& 0	& 1	& 0	\\
		2,2		& 0	& 0	& 0	& 0
	\end{bNiceArray}
	$}
	\\
\end{tblr}
  	\label{matrix2,2old}
	\end{center}


\section{Adding Features and Biases} \label{appfeatbiases}

\subsection{Features} 

All of the results that appeared in Section \ref{charPerm}
can be adapted for the case where the feature dimension of the layer spaces is greater than one.
If the feature dimension of a layer space is now $d$, then
instead of the space being
$(\mathbb{R}^{n})^{\otimes k}$, 
it is now 
$(\mathbb{R}^{n})^{\otimes k} \otimes \mathbb{R}^d$.
Consequently, we can adapt Theorem \ref{equivbasis}
to obtain the following result.
\begin{theorem} \label{equivbasislpluskwithfeatures}
	For all non-negative integers $l, k$ and positive integers $n, d_k$ and $d_l$,
	we have that
	\begin{equation}
		\{X_{\pi,i,j} \mid \pi \in \Pi_{l+k,n}, i \in [d_l], j \in [d_k] \}
	\end{equation}
	is a basis of 
	$\Hom_{S_n}(
	(\mathbb{R}^{n})^{\otimes k} \otimes \mathbb{R}^{d_k},
	(\mathbb{R}^{n})^{\otimes l} \otimes \mathbb{R}^{d_l}
	)$,
	where now
	\begin{equation} \label{equivbasiselementlpluskfeatures}
		X_{\pi,i,j} \coloneqq \sum_{(I,J) \in O((I_{\pi},J_{\pi}))} E_{I,i,J,j}
	\end{equation}
	Here, $E_{I,i,J,j}$ is a matrix unit in 
	$\Hom(
	(\mathbb{R}^{n})^{\otimes k} \otimes \mathbb{R}^{d_k},
	(\mathbb{R}^{n})^{\otimes l} \otimes \mathbb{R}^{d_l}
	)$,
	for all tuples
	$I \in [n]^l$ and $J \in [n]^k$, 
	and for all $i \in [d_l], j \in [d_k]$.

	Hence, we have that
	\begin{equation}
		\dim 
		\Hom_{S_n}(
		(\mathbb{R}^{n})^{\otimes k} \otimes \mathbb{R}^{d_k},
		(\mathbb{R}^{n})^{\otimes l} \otimes \mathbb{R}^{d_l}
		)
		= 
		d_k d_l
		\sum_{t = 1}^{n} 
			\begin{Bmatrix}
				l+k\\
				t 
			\end{Bmatrix}
		=
		d_k d_l
		\Bell(l+k, n)
	\end{equation}
\end{theorem}

\subsection{Biases} 

Including bias terms in the layer functions of a permutation equivariant neural network is harder, but it can be done.
We want to include such a term but still keep the entire layer function permutation equivariant.
Since the addition of two permutation equivariant functions 
is permutation equivariant, 
we can include bias terms by adding to each linear layer function
	\begin{equation}
		\phi: 
		((\mathbb{R}^{n})^{\otimes k}, \rho_{k}) 
		\rightarrow 
		((\mathbb{R}^{n})^{\otimes l}, \rho_{l}) 
	\end{equation}
	a permutation equivariant bias function
	$\beta : 
		((\mathbb{R}^{n})^{\otimes k}, \rho_{k}) 
		\rightarrow 
		((\mathbb{R}^{n})^{\otimes l}, \rho_{l})$
that is constant on all elements of 
$(\mathbb{R}^{n})^{\otimes k}$, that is,
	\begin{equation}
		\beta_l(v) = c \text{ for all } v \in (\mathbb{R}^{n})^{\otimes k}
	\end{equation}
for some constant element $c \in (\mathbb{R}^{n})^{\otimes l}$.
Since the bias function $\beta$ is permutation equivariant, it needs to satisfy
	\begin{equation} \label{biasequiv}
		c = \rho_l(g)c
	\end{equation}
for all $g \in S_n$ and $c \in (\mathbb{R}^{n})^{\otimes l}$.
Given that any $c \in (\mathbb{R}^{n})^{\otimes l}$ satisfying (\ref{biasequiv}) can be viewed as an element of $\Hom_{S_n}((\mathbb{R}^{n})^{\otimes 0}, (\mathbb{R}^{n})^{\otimes l})$, to find the matrix form of $c$, all we need to do is find the basis elements of $\Hom_{S_n}((\mathbb{R}^{n})^{\otimes 0}, (\mathbb{R}^{n})^{\otimes l})$.
But we saw how to do this in Section \ref{charPerm} ---
we can set $k = 0$, consider the vector space $P_0^l(n)$, and then apply Theorem \ref{weightmatclass}
to obtain the bias terms.

\section{A Generalisation to Layer Functions that are Equivariant to a Product of Symmetric Groups} \label{generalisationproduct}

We can generalise the results presented in Section \ref{charPerm}
to construct neural networks that are equivariant to a product of symmetric groups 
$S_{n_1} \times \dots \times S_{n_m}$.
The layers in this case are the (external) tensor product of some $k_r$-order tensor power of the permutation representation $\mathbb{R}^{n_r}$ of the symmetric group $S_{n_r}$, for each $r \in [m]$.
These permutation equivariant neural networks model local symmetries in data, since we can think of there being a total of $n_1 + \dots + n_m$ objects, but the network is constructed such that only the symmetries in each 
set of $n_r$ objects is respected.

Given groups $G$ and $H$, and representations $\rho_V : G \rightarrow GL(V)$ and $\rho_W : H \rightarrow GL(W)$, we can define the external tensor product representation of the direct product group $G \times H$, namely
\begin{equation}
	\rho_{V \boxtimes W} : G \times H \rightarrow GL(V \boxtimes W)
\end{equation}
as follows: as a vector space, $V \boxtimes W = V \otimes W$, and
\begin{equation}
	\rho_{V \boxtimes W}(g,h) = \rho_V(g) \otimes \rho_W(h) \text{ for all } g \in G \text{ and } h \in H.
\end{equation}
This definition can be extended in the obvious way to a direct product of any number of groups with their accompanying representations.


It can be shown that a map $f : V_1 \boxtimes W_1 \rightarrow V_2 \boxtimes W_2$ is $G \times H$-equivariant if and only if it is the tensor product $f_1 \otimes f_2$ of a $G$-equivariant map $f_1 : V_1 \rightarrow W_1$ and an $H$-equivariant map $f_2 : V_2 \rightarrow W_2$, and so, for linear maps, we have that
\begin{equation} \label{externalequiv}
	\Hom_{G \times H}(V_1 \boxtimes W_1, V_2 \boxtimes W_2)
	\cong
	\Hom_{G}(V_1, V_2) \otimes \Hom_{H}(W_1, W_2)
\end{equation}
As a result, we can generalise the construction presented in Section \ref{charPerm}
to find a basis of
\begin{equation} \label{genericHomspace}
	\Hom_{S_{n_1} \times \dots \times S_{n_m}}
	((\mathbb{R}^{n_1})^{\otimes {k_1}} \boxtimes \dots \boxtimes (\mathbb{R}^{n_m})^{\otimes {k_m}}, 
	 (\mathbb{R}^{n_1})^{\otimes {l_1}} \boxtimes \dots \boxtimes (\mathbb{R}^{n_m})^{\otimes {l_m}})
\end{equation}
We start by forming the tensor product vector space $P_{k_1}^{l_1}(n_1) \otimes \dots \otimes P_{k_m}^{l_m}(n_m)$, which has an orbit basis
consisting of $m$-length tuples of orbit basis elements from each of the individual vector spaces $P_{k_r}^{l_r}(n_r)$ in the direct product. 

We represent such tuples by placing the orbit basis diagrams from each of the individual vector spaces side-by-side, putting a red demarcation line between each adjacent pair: for example, the diagram
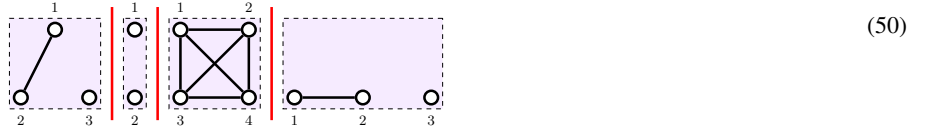
\begin{equation} \label{orbitexternalprodexample}
\scalebox{0.6}{\begin{tikzpicture}
	\begin{pgfonlayer}{nodelayer}
		\node [style=none] (0) at (2.5, 1) {$1$};
		\node [style=node] (1) at (1, -3) {};
		\node [style=node] (2) at (2.5, 0) {};
		\node [style=node] (3) at (4, -3) {};
		\node [style=none] (4) at (1, -4) {$2$};
		\node [style=none] (5) at (4, -4) {$3$};
		\node [style=none] (6) at (0.5, 0.5) {};
		\node [style=none] (7) at (0.5, -3.5) {};
		\node [style=none] (8) at (4.5, -3.5) {};
		\node [style=none] (9) at (4.5, 0.5) {};
		\node [style=none] (10) at (6, 1) {$1$};
		\node [style=node] (11) at (6, 0) {};
		\node [style=node] (12) at (6, -3) {};
		\node [style=none] (13) at (6, -4) {$2$};
		\node [style=none] (14) at (5.5, 0.5) {};
		\node [style=none] (15) at (5.5, -3.5) {};
		\node [style=none] (16) at (6.5, -3.5) {};
		\node [style=none] (17) at (6.5, 0.5) {};
		\node [style=none] (18) at (5, 1) {};
		\node [style=none] (19) at (5, -4) {};
		\node [style=none] (20) at (8, 1) {$1$};
		\node [style=node] (21) at (8, 0) {};
		\node [style=node] (22) at (8, -3) {};
		\node [style=node] (23) at (11, 0) {};
		\node [style=node] (24) at (11, -3) {};
		\node [style=none] (25) at (11, 1) {$2$};
		\node [style=none] (26) at (8, -4) {$3$};
		\node [style=none] (27) at (11, -4) {$4$};
		\node [style=none] (28) at (7.5, 0.5) {};
		\node [style=none] (29) at (7.5, -3.5) {};
		\node [style=none] (30) at (11.5, -3.5) {};
		\node [style=none] (31) at (11.5, 0.5) {};
		\node [style=none] (32) at (7, 1) {};
		\node [style=none] (33) at (7, -4) {};
		\node [style=none] (34) at (13, -4) {$1$};
		\node [style=node] (35) at (13, -3) {};
		\node [style=node] (36) at (19, -3) {};
		\node [style=node] (37) at (16, -3) {};
		\node [style=none] (38) at (16, -4) {$2$};
		\node [style=none] (39) at (19, -4) {$3$};
		\node [style=none] (40) at (12.5, -3.5) {};
		\node [style=none] (41) at (19.5, -3.5) {};
		\node [style=none] (42) at (12.5, 0.5) {};
		\node [style=none] (43) at (19.5, 0.5) {};
		\node [style=none] (44) at (12, 1) {};
		\node [style=none] (45) at (12, -4) {};
	\end{pgfonlayer}
	\begin{pgfonlayer}{edgelayer}
		\draw [style={dash_2}] (14.center)
			 to (15.center)
			 to (16.center)
			 to (17.center)
			 to cycle;
		\draw [style={red_thick}] (19.center) to (18.center);
		\draw [style={dash_2}] (28.center)
			 to (29.center)
			 to (30.center)
			 to (31.center)
			 to cycle;
		\draw [style=thick] (21) to (24);
		\draw [style=thick] (22) to (23);
		\draw [style=thick] (22) to (21);
		\draw [style=thick] (22) to (24);
		\draw [style=thick] (24) to (23);
		\draw [style=thick] (23) to (21);
		\draw [style={red_thick}] (33.center) to (32.center);
		\draw [style={red_thick}] (45.center) to (44.center);
		\draw [style={dash_2}] (9.center)
			 to (6.center)
			 to (7.center)
			 to (8.center)
			 to cycle;
		\draw [style=thick] (1) to (2);
		\draw [style={dash_2}] (43.center)
			 to (42.center)
			 to (40.center)
			 to (41.center)
			 to cycle;
		\draw [style=thick] (35) to (37);
	\end{pgfonlayer}
\end{tikzpicture}}
\end{equation}
is an orbit basis element in $P_{2}^{1}(n_1) \otimes P_{1}^{1}(n_2) \otimes P_{2}^{2}(n_3) \otimes P_{3}^{0}(n_4)$, for any $n_1, n_2, n_3, n_4 \in \mathbb{Z}_{\geq 1}$.

Then the standard matrix basis of the $\Hom$-space given in (\ref{genericHomspace}) is simply the images of those orbit basis elements of 
\begin{equation}
	P_{k_1}^{l_1}(n_1) \otimes \dots \otimes P_{k_m}^{l_m}(n_m)
\end{equation}
whose $m$-length tuples, for each index $r \in [m]$, consist of only those orbit basis elements of $P_{k_r}^{l_r}(n_r)$ that correspond to a set partition in $\Pi_{l_r + k_r}$ having at most $n_r$ blocks,
under the map
\begin{equation} \label{externalprodmap}
	\Phi_{(k_1, \dots, k_m),(n_1, \dots, n_m)}^{(l_1, \dots, l_m)}
	\coloneqq
	\Phi_{k_1,n_1}^{l_1} \otimes \dots \otimes \Phi_{k_m,n_m}^{l_m}
\end{equation}
that maps $P_{k_1}^{l_1}(n_1) \otimes \dots \otimes P_{k_m}^{l_m}(n_m)$ onto the $\Hom$-space given in (\ref{genericHomspace}).

The map (\ref{externalprodmap}) is well defined because the $\Hom$-space given in (\ref{genericHomspace}) is isomorphic to
\begin{equation} \label{tensorgenericHomspace}
	\bigotimes_{r=1}^{m} \Hom_{S_{n_r}}((\mathbb{R}^{n_r})^{\otimes k_r}, (\mathbb{R}^{n_r})^{\otimes l_r})
\end{equation}
by a repeated application of (\ref{externalequiv}).

It is clear that the standard matrix basis of the $\Hom$-space given in (\ref{genericHomspace}) will be the tensor (Kronecker) product of the standard matrix basis elements of the individual $\Hom$-spaces given in (\ref{tensorgenericHomspace}), and so the dimension of this space is
\begin{equation} \label{dimensionsubspace}
	\prod_{r = 1}^{m} \Bell(l_r + k_r, n_r)
\end{equation}
Clearly, by setting $l_r = 0$ for all $r \in [m]$ in (\ref{dimensionsubspace}), the $S_{n_1} \times \dots \times S_{n_m}$--invariant maps have dimension $\prod_{r = 1}^{m} \Bell(k_r, n_r)$.

In other words, to find the basis of the $\Hom$-space given in (\ref{genericHomspace}), all we need to do is to consider all possible side-by-side combinations of the appropriate orbit basis diagrams for each of the individual vector spaces $P_{k_r}^{l_r}(n_r)$ (those in bijective correspondence with $\Pi_{l_r + k_r, n_r}$), and calculate, for each such combination, each constituent's image under its respective map $\Phi_{k_r, n_r}^{l_r}$, and then finally calculate the tensor (Kronecker) product of the resulting matrices.
This makes for quite a powerful diagrammatic method for finding such a basis.

We could also extend (\ref{genericHomspace}) by adding a dimension of features for each of the layers; in this case, the $\Hom$-space under consideration becomes
\begin{equation}
	\Hom_{S_{n_1} \times \dots \times S_{n_m}}
	(
	(\mathbb{R}^{n_1})^{\otimes {k_1}} \boxtimes \mathbb{R}^{d_{k_1}} \boxtimes \dots \boxtimes (\mathbb{R}^{n_m})^{\otimes {k_m}} \boxtimes \mathbb{R}^{d_{k_m}}, 
	(\mathbb{R}^{n_1})^{\otimes {l_1}} \boxtimes \mathbb{R}^{d_{l_1}} \boxtimes \dots \boxtimes (\mathbb{R}^{n_m})^{\otimes {l_m}} \boxtimes \mathbb{R}^{d_{l_m}}
	)
\end{equation}
and so, using the result of Theorem \ref{equivbasislpluskwithfeatures} together with (\ref{dimensionsubspace}), we see that the overall dimension of this $\Hom$-space is
$\prod_{r = 1}^{m} d_{k_r}d_{l_r}\Bell(l_r + k_r, n_r)$.


\subsection{An Example Recovering the Result of
\citet{Hartford}} \label{hartfordexample}

One immediate consequence of our work is that we can recover the result of \citet{Hartford} through our method and 
give an
explanation 
as to why their result holds
that is grounded in the language of the partition algebras.

In effect, Hartford et al. obtain a standard matrix basis of the $\Hom$-space given in (\ref{genericHomspace}) for the case where $l_r = k_r = 1$ and $n_r \geq 2$ for all $r \in [m]$, that is, for
\begin{equation} \label{Hartford}
	\Hom_{S_{n_1} \times \dots \times S_{n_m}}
	(\mathbb{R}^{n_1} \boxtimes \dots \boxtimes \mathbb{R}^{n_m},
	\mathbb{R}^{n_1} \boxtimes \dots \boxtimes \mathbb{R}^{n_m})
\end{equation}
For example, the weight matrix for
$\Hom_{S_2 \times S_2 \times S_2}
(\mathbb{R}^2 \boxtimes \mathbb{R}^2 \boxtimes \mathbb{R}^2,
\mathbb{R}^2 \boxtimes \mathbb{R}^2 \boxtimes \mathbb{R}^2)$ 
is
\begin{equation}
	\NiceMatrixOptions{code-for-first-row = \scriptstyle \color{blue},
                   	   code-for-first-col = \scriptstyle \color{blue}
	}
	\begin{bNiceArray}{*{2}{c}|*{2}{c}|*{2}{c}|*{2}{c}}[first-row,first-col]
	\RowStyle[cell-space-limits=3pt]{\rotate}
				& 1,1,1 	& 1,1,2		& 1,2,1 	& 1,2,2		& 2,1,1 	& 2,1,2		& 2,2,1 	& 2,2,2	 \\
		1,1,1		& \lambda_1 	& \lambda_2 	& \lambda_3 	& \lambda_4	& \lambda_5	& \lambda_6	& \lambda_7	& \lambda_8	\\
		1,1,2		& \lambda_2 	& \lambda_1 	& \lambda_4 	& \lambda_3 	& \lambda_6	& \lambda_5	& \lambda_8	& \lambda_7	\\
		\hline
		1,2,1		& \lambda_3 	& \lambda_4 	& \lambda_1  	& \lambda_2 	& \lambda_7	& \lambda_8	& \lambda_5	& \lambda_6 	\\
		1,2,2		& \lambda_4 	& \lambda_3 	& \lambda_2 	& \lambda_1 	& \lambda_8	& \lambda_7	& \lambda_6	& \lambda_5	\\
		\hline
		2,1,1		& \lambda_5 	& \lambda_6 	& \lambda_7 	& \lambda_8 	& \lambda_1	& \lambda_2 	& \lambda_3	& \lambda_4	\\
		2,1,2		& \lambda_6 	& \lambda_5 	& \lambda_8 	& \lambda_7 	& \lambda_2 	& \lambda_1	& \lambda_4	& \lambda_3	\\
		\hline
		2,2,1		& \lambda_7 	& \lambda_8 	& \lambda_5 	& \lambda_6 	& \lambda_3	& \lambda_4	& \lambda_1 	& \lambda_2 	\\
		2,2,2		& \lambda_8 	& \lambda_7 	& \lambda_6 	& \lambda_5 	& \lambda_4	& \lambda_3	& \lambda_2 	& \lambda_1 	
	\end{bNiceArray}
\end{equation}
for weights $\lambda_1, \lambda_2, \dots, \lambda_8 \in \mathbb{R}$, which
is obtained by mapping the orbit basis of the tensor product of partition algebras $P_1^1(2) \otimes P_1^1(2) \otimes P_1^1(2)$ onto 
$\Hom_{S_2 \times S_2 \times S_2}
(\mathbb{R}^2 \boxtimes \mathbb{R}^2 \boxtimes \mathbb{R}^2,
\mathbb{R}^2 \boxtimes \mathbb{R}^2 \boxtimes \mathbb{R}^2)$ 
via the map $\Phi_{(1,1,1),(2,2,2)}^{(1,1,1)}$ as defined in (\ref{externalprodmap}).

It is clear that, in the general case, the weight matrix can be obtained by mapping the orbit basis of the tensor product of partition algebras $P_1^1(n_1) \otimes P_1^1(n_2) \otimes \dots \otimes P_1^1(n_m)$
onto the space given in (\ref{Hartford}) via the map $\Phi_{(1,1, \dots, 1),(n_1,n_2, \dots, n_m)}^{(1,1, \dots, 1)}$.


\section{PyTorch Implementation of Permutation Equivariant Weight Matrices}

In this section, we provide a PyTorch implementation of all of the permutation equivariant weight matrices for any symmetric group $S_n$ and for all possible tensor power spaces of $\mathbb{R}^n$.
The weight matrices will be instances of the class \texttt{SymmetricGrpEquivLinear}.

\begin{lstlisting}
# symmequiv.py

import torch
import torch.nn as nn
import torch.optim as optim

from .lib import symmpartitions

class SymmetricGrpEquivLinear(nn.Module):
    """
    Creates a trainable S_n-equivariant linear layer 
    (R^n)^{otimes k} rightarrow (R^n)^{otimes l} using
    the orbit basis for the partition vector space P_k^l(n)
    to create the weight matrices.
    
    dim_n: the dimension of the space, i.e the n in S_n
    order_k: the tensor power k
    order_l: the tensor power l
    """

    def __init__(self, dim_n: int, order_k: int, order_l: int):
        super().__init__()
        self.n = dim_n
        self.k = order_k
        self.l = order_l
        
        #Trick used below to get everything on same device!
        self.dummy_param = nn.Parameter(torch.empty(0))
        
        self.basis_set_matrices, self.num_weights = self.__basis_set_matrices_generation()

        self.weights = nn.ParameterList([])
        for i in range(self.num_weights): 
            self.weights.append(nn.Parameter(torch.randn(())))

    def print_weights(self) -> None:
        for i in range(len(self.weights)):
            print(self.weights[i])

    def __basis_set_matrices_generation(self):
        part_lst_by_indices = symmpartitions.set_partition_weight_matrices_by_indices(
                dim=self.n, order_k=self.k, order_l=self.l
                )
        matrices = []
        for val in part_lst_by_indices:
            mat = torch.zeros(pow(self.n, self.l),pow(self.n, self.k))
            for ind in val:
                mat[ind[0]][ind[1]] = 1
            matrices.append(mat)
        return matrices, len(matrices)

    def forward(self, X):
        #Trick to get everything on the same device for training etc.
        device = self.dummy_param.device

        #Move all weight matrices onto device first before performing calculations.
        for i in range(len(self.basis_set_matrices)):
            self.basis_set_matrices[i] = self.basis_set_matrices[i].to(device)

        #Move weight_matrix onto device before calculating its value.
        weight_matrix = torch.zeros(pow(self.n, self.l),pow(self.n, self.k)).to(device)
    
        for weight_index, mat in enumerate(self.basis_set_matrices):
            weight_matrix += mat * self.weights[weight_index]
        
        linear = torch.einsum('ij,kj->ki', weight_matrix, X)    # allows for batch processing
        return linear

\end{lstlisting}

\begin{lstlisting}
# symmpartitions.py 

import itertools
import more_itertools

from .setpartitions import *

def set_partitions(order_k: int, order_l: int, max_blocks: int) -> list:
    """
    Calculates a list of all set partitions having at most max_blocks blocks.
    Here, we assume that the set partition corresponds to a set partition diagram
    having order_k nodes at the bottom and order_l nodes at the top.

    Returns
    -------
    list
    """

    # Maximum number of blocks cannot be more than the sum of the orders!
    total_order = order_k + order_l
    if max_blocks > total_order:
        max_blocks = total_order

    res = []
    for i in range(1, order_k + order_l + 1):
        res.append(i)
    collection = tuple(res)

    set_part = []
    for k in range(1, max_blocks + 1):
        s1 = more_itertools.set_partitions(collection,k)
        for part in s1:
            set_part.append(part)

    return set_part

def set_partition_weight_matrices_by_indices(dim: int, order_k: int, order_l: int) -> list:
    """
    Returns a list consisting of lists of indices, where each 
    list of indices corresponds to where a partition spanning matrix is non-zero,
    and the number of set partitions that appeared in the calculation.

    Returns
    -------
    list
    """
    lst = []
    set_parts = set_partitions(order_k, order_l, dim)
    for set_part in set_parts:
        part_indices = set_part_diag_basis_indices_list_ord(set_part, dim)
        mat_indices = mat_indices_list(part_indices, dim, order_k, order_l)
        lst.append(mat_indices)
    return lst
\end{lstlisting}

\begin{lstlisting}
# setpartitions.py

import itertools
import more_itertools

def convert_tuple_to_matrix_index(indices: tuple, dim: int, order_k: int) -> int:
    """
    Helper function that converts a tuple with indices
    (i_1, ..., i_k) where i_j is an element of {1, ..., dim}
    that indexes a tensor to the equivalent index that indexes
    the same tensor in matrix form.

    Returns
    -------
    int
    """
    assert(len(indices) == order_k)
    total = 0
    for index in indices:
        total += pow(dim, order_k - 1)*(index - 1)
        order_k -=1
    return total 

def convert_int_to_tuple(num: int, dim: int, order_k: int) -> list: 
    """
    The reverse function of convert_tuple_to_matrix_index().
    Converts an integer back into a tuple of length order_k
    according to the base = dim

    Returns
    -------
    list
    """
    lst = []
    for i in range(order_k):
        if num < dim:
            lst.append(int(num+1))
            num = 0
        else:
            val = num % dim
            lst.append(int(val+1))
            sub = pow(dim, order_k - (order_k + i))*val
            num -= sub 
            num //= dim
    lst.reverse()
    return lst

def set_part_diag_basis_indices_list_ord(lst: list, dim: int) -> list:
    """
    Takes a pattern corresponding to a set partition
    and calculates all indices that match it 
    in the orbit basis, based on the value of dim (== n).
    
    Note that this returns a list of the form [I,J], i.e both the row
    and the column tuples, that can be divided appropriately
    by the function mat_indices_list according to the orders k, l.

    Returns
    -------
    list
    """
    num_blocks = len(lst)

    # Total sum of the tensor power orders
    total_tensor_orders = sum(len(sublst) for sublst in lst)

    # Enumerate the blocks given in the input lst
    blocks = {i: list(lst[i-1]) for i in range(1, num_blocks+1)}
    
    indices_lst = []    
    for i in range(pow(dim, num_blocks)):
        block_labels = convert_int_to_tuple(i,dim,num_blocks)  

        # Block labels must all be different 
        if len(block_labels) != len(set(block_labels)): 
            continue

        # vertex_labels consists of pairs of the form
        # diagram vertex label : value that appears in the [I,J] list for that vertex
        vertex_labels = {}
        for j in range(num_blocks):
            I_J_val = block_labels[j]
            block = blocks[j+1]
            for k in range(len(block)):
                vertex_labels[block[k]] = I_J_val
        list_tup = [vertex_labels[i] for i in range(1, total_tensor_orders+1)] 
        indices_lst.append(list_tup)
    
    return indices_lst

def mat_indices_list(indices_lst: list, dim: int, order_k: int, order_l: int) -> list:
    """
    Converts a list of tuple indices corresponding to a set partition, 
    for a given dim and orders k,l into their equivalent matrix index form.

    Returns
    -------
    list
    """

    assert(len(indices_lst[0]) == order_k + order_l)
   
    lst = []
    for indices in indices_lst:
        row_indices = indices[:order_l]
        col_indices = indices[order_l:]
        row_index = convert_tuple_to_matrix_index(row_indices, dim, order_l)
        col_index = convert_tuple_to_matrix_index(col_indices, dim, order_k)
        lst.append([row_index, col_index])
    return lst
\end{lstlisting}


\begin{example}
We obtain the basis of matrices that determine the permutation equivariant weight matrices
that appear in Examples \ref{114example} and \ref{lessBell}
using the code that is provided above.

\begin{lstlisting}
# example.py: displays basis matrices that are generated using our linear layers.

from nn.symmequiv import SymmetricGrpEquivLinear

symm_model_4_1_1 = SymmetricGrpEquivLinear(dim_n = 4, order_k = 1, order_l = 1)
print("Number of weights: ", len(symm_model_4_1_1.basis_set_matrices))
for basis_matrix in symm_model_4_1_1.basis_set_matrices:
    print(basis_matrix)

symm_model_2_2_2 = SymmetricGrpEquivLinear(dim_n = 2, order_k = 2, order_l = 2)
print("Number of weights: ", len(symm_model_2_2_2.basis_set_matrices))
for basis_matrix in symm_model_2_2_2.basis_set_matrices:
    print(basis_matrix)

\end{lstlisting}

\begin{verbatim}
% python3 example.py

Number of weights:  2
tensor([[1., 0., 0., 0.],
        [0., 1., 0., 0.],
        [0., 0., 1., 0.],
        [0., 0., 0., 1.]])
tensor([[0., 1., 1., 1.],
        [1., 0., 1., 1.],
        [1., 1., 0., 1.],
        [1., 1., 1., 0.]])

Number of weights:  8
tensor([[1., 0., 0., 0.],
        [0., 0., 0., 0.],
        [0., 0., 0., 0.],
        [0., 0., 0., 1.]])
tensor([[0., 0., 0., 0.],
        [0., 0., 0., 1.],
        [1., 0., 0., 0.],
        [0., 0., 0., 0.]])
tensor([[0., 0., 0., 1.],
        [0., 0., 0., 0.],
        [0., 0., 0., 0.],
        [1., 0., 0., 0.]])
tensor([[0., 0., 0., 0.],
        [1., 0., 0., 0.],
        [0., 0., 0., 1.],
        [0., 0., 0., 0.]])
tensor([[0., 1., 0., 0.],
        [0., 0., 0., 0.],
        [0., 0., 0., 0.],
        [0., 0., 1., 0.]])
tensor([[0., 0., 0., 0.],
        [0., 0., 1., 0.],
        [0., 1., 0., 0.],
        [0., 0., 0., 0.]])
tensor([[0., 0., 0., 0.],
        [0., 1., 0., 0.],
        [0., 0., 1., 0.],
        [0., 0., 0., 0.]])
tensor([[0., 0., 1., 0.],
        [0., 0., 0., 0.],
        [0., 0., 0., 0.],
        [0., 1., 0., 0.]])
\end{verbatim}
\end{example}


\end{document}